\documentclass{article}
\usepackage[utf8]{inputenc}

\usepackage{authblk} 

\usepackage{hyperref}
\usepackage{amsmath, amssymb, amsthm, mathrsfs} 
\usepackage{mleftright} 
    \mleftright
\usepackage[algo2e,ruled]{algorithm2e} 
\usepackage{microtype} 
\usepackage{booktabs} 
\usepackage{enumitem} 
\usepackage{nicefrac} 
\usepackage{tikz} 
    \usetikzlibrary{external} 
\usepackage{natbib} 
\usepackage{soul} 
\usepackage{natbib} 

\newcommand{\cB}{\mathcal{B}}

\newcommand{\cF}{\mathcal{F}}
\newcommand{\cG}{\mathcal{G}}

\newcommand{\cL}{\mathcal{L}}

\newcommand{\cV}{\mathcal{V}}
\newcommand{\cW}{\mathcal{W}}
\newcommand{\cX}{\mathcal{X}}

\newcommand{\cZ}{\mathcal{Z}}

\newcommand{\E}{\mathbb{E}}

\newcommand{\I}{\mathbb{I}}

\newcommand{\N}{\mathbb{N}}

\renewcommand{\P}{\mathbb{P}}

\newcommand{\R}{\mathbb{R}}

\DeclareMathOperator*{\argmin}{argmin}

\DeclareMathOperator*{\supp}{supp}
\newcommand{\dif}{\,\mathrm{d}}
\newcommand{\e}{\varepsilon}

\newcommand{\tht}{\vartheta}
\newcommand{\lrb}[1]{\left(#1\right)}
\newcommand{\brb}[1]{\bigl(#1\bigr)}
\newcommand{\Brb}[1]{\Bigl(#1\Bigr)}
\newcommand{\bbrb}[1]{\biggl(#1\biggr)}
\newcommand{\lsb}[1]{\left[#1\right]}
\newcommand{\bsb}[1]{\bigl[#1\bigr]}
\newcommand{\Bsb}[1]{\Bigl[#1\Bigr]}
\newcommand{\bbsb}[1]{\biggl[#1\biggr]}
\newcommand{\Bbsb}[1]{\Biggl[#1\Biggr]}
\newcommand{\lcb}[1]{\left\{#1\right\}}
\newcommand{\bcb}[1]{\bigl\{#1\bigr\}}
\newcommand{\Bcb}[1]{\Bigl\{#1\Bigr\}}

\newcommand{\bce}[1]{\bigl\lceil#1\bigr\rceil}

\newcommand{\labs}[1]{\left\lvert#1\right\rvert}
\newcommand{\babs}[1]{\bigl\lvert#1\bigr\rvert}
\newcommand{\Babs}[1]{\Bigl\lvert#1\Bigr\rvert}

\newcommand{\lno}[1]{\left\lVert#1\right\rVert}

\newcommand{\wt}{\widetilde}
\newcommand{\s}{\subset}
\newcommand{\m}{\setminus}
\newcommand{\iop}{\infty}
\renewcommand{\l}{\ldots}

\newcommand{\nn}{Nearest Neighbor}
\newcommand{\nna}{Nearest Neighbor algorithm}

\newcommand{\smn}{\psi}
\newcommand{\Smn}{\Psi}
\newcommand{\leb}{Lebesgue}
\newcommand{\besi}{Besicovitch}
\newcommand{\wtil}{\wt{\eta}}
\newcommand{\aaa}[1]{#1}

\newtheorem{lemma}{Lemma}
\newtheorem{proposition}{Proposition}
\newtheorem{corollary}{Corollary}

\newtheorem{theorem}{Theorem}
    \theoremstyle{definition}

\newtheorem{assumption}{Assumption}
\newtheorem{definition}{Definition}

\newtheorem{sett}{Setting}
\newtheorem{example}{Example}

\title{A Nearest Neighbor Characterization of \leb{} Points in Metric Measure Spaces}

\author[1]{Tommaso R. Cesari}

\affil[1]{Toulouse School of Economics (TSE), Toulouse, France

        \& Artificial and Natural Intelligence Toulouse Institute (ANITI)
        
         \href{mailto:tom.cesari@univ-toulouse.fr}{email: tom.cesari@univ-toulouse.fr}}
        
\author[2]{Roberto Colomboni}

\affil[2]{Istituto Italiano di Tecnologia, Genova, Italy
        
        \& Universit\`a degli Studi di Milano, Milano, Italy

        \href{mailto:roberto.colomboni@unimi.it}{email:
        roberto.colomboni@unimi.it}}

\begin{document}

\maketitle

\textbf{keywords:} Nearest Neighbor algorithms, geometric measure theory

\begin{abstract}
The property of almost every point being a \leb{} point has proven to be crucial for the consistency of several classification algorithms based on nearest neighbors.
We characterize \leb{} points in terms of a $1$-\nn{} regression algorithm for pointwise estimation, fleshing out the role played by tie-breaking rules in the corresponding convergence problem. 
We then give an application of our results, proving the convergence of the risk of a large class of $1$-\nn{} classification algorithms in general metric spaces where almost every point is a \leb{} point.
\end{abstract}

\section{Introduction}

A point $x$ in a metric space is a \emph{\leb{} point} for a function $f$ with respect to a locally-finite measure $\mu$ if
\[
    \frac{1}{\mu\brb{\bar{B}_r(x)}} \int_{\bar{B}_r(x)} \babs{ f(x') - f(x) } \dif \mu(x') \to 0\;, \qquad r\to 0^+
\]
where $\bar{B}_r(x)$ is the closed ball of radius $r$ centered at $x$.
\leb{} points are an integral generalization of continuity points. 
They originally found applications in Fourier analysis: 
\cite{lebesgue1905recherches} (resp., \cite{fatou1906series}) showed that the Fourier series of an integrable function $f$ is Ces\`aro-summable (resp., non-tangentially Abel-summable) to $f$ at all its \leb{} points.
They also find applications in harmonic analysis \cite[Theorem 1.25]{stein1971introduction}, wavelet and spline theory \cite[Theorem 2.1 and Corollary 2.2]{kelly1994pointwise}, and are a central concept in geometric measure theory, both in $\R^d$ \citep{evans2015measure,federer2014geometric,maggi2012sets,mattila1999geometry} and in general metric spaces \citep{cheeger1999differentiability,kinnunen2002lebesgue,kinnunen2008lebesgue,bjorn2010lebesgue}.
The most famous result on \leb{} points is probably the celebrated \leb{}--\besi{} differentiation theorem which states that for any Radon measure $\mu$ on a Euclidean space, $\mu$-almost every point is a \leb{} point for all $f\in\cL^1_{\mathrm{loc}}(\mu)$ (see, e.g., \cite[Theorems~1.32-1.33]{evans2015measure}).
\cite{preiss1979invalid} showed that this result does not hold in general metric spaces and characterized those in \aaa{which it does} \citep{preiss1983dimension}. 
These spaces include finite-dimensional Banach spaces \citep{loeb2006microscopic}, locally-compact separable ultrametric spaces \cite[Theorem 9.1]{simmons2012conditional}, separable Riemannian manifolds \cite[Theorem 9.1]{simmons2012conditional}, and (straightforwardly) countable spaces.

The \leb{}--\besi{} differentiation theorem found several applications in classification, regression and density-estimation problems with \nna{}s and variants thereof \citep{abraham2006kernel,biau2015lectures}.
To the best of our knowledge, \cite{devroye1981inequality} was the first to show a connection between the \leb{}--\besi{} differentiation theorem in $\R^d$ and the convergence of the risk of the \nn{} classifier in which \emph{ties are broken lexicographically} (i.e., when ties are broken by taking the sample with smallest index).
With the same tie-breaking rule, \cite{devroye1981almost} showed that the \leb{}--\besi{} differentiation theorem plays a crucial role in regression problems with \nn{} and kernel algorithms.
\cite{cerou2006nearest} proved that, if \emph{ties are broken uniformly at random}, the $k_m$-\nn{} classifier is consistent in any Polish metric space in which the \leb{}--\besi{} differentiation theorem holds (see also~\citep{forzani2012consistent,chaudhuri2014rates}). 
\aaa{Finally, it was shown that it is possible to combine compression techniques with $1$-\nn{} classification in order achieve consistency in essentially separable metric spaces, even if the \leb{}--\besi{} differentiation theorem does not hold \citep{hanneke2019universal}.
Recently, \cite{gyorfi2020universal} achieved the same consistency result through a simpler algorithm that does not rely on compression.}

\paragraph{Our contributions.}

The main purpose of this paper is to give a characterization (Theorem~\ref{t:sum-up} and Corollary~\ref{c:powa}) of \leb{} points in terms of an $\cL^1$-convergence property of a \nna{}.
More precisely, take a bounded measurable function $\eta$ defined on a metric space $(\cX,d)$, a random i.i.d.\ sample $X_1, \l, X_m$ on $\cX$, and an arbitrary point $x\in\cX$ in the support of the distribution of the $X_k$'s.
The algorithm evaluates $\eta$ at a point 
\[
    X^x_m \in \argmin_{X' \in \{X_1, \l, X_m\}} d(x,X')
\]
with the goal of approximating $\eta(x)$.

First, we prove that, if $\eta(X^x_m)$ converges to $\eta(x)$ in $\cL^1$, then $x$ is a \leb{} point for $\eta$, regardless of how ties in the definition of $X^x_m$ are broken (Theorem~\ref{t:nn-then-leb}).

Vice versa, it is known that, if ties are broken lexicographically and $x$ is a \leb{} point for $\eta$, then $\eta(X^x_m)$ converges to $\eta(x)$ in $\cL^1$ \citep{devroye1981inequality}.
By means of a novel technique (Theorem~\ref{t:victory-spheres}), we extend this result allowing more general tie-breaking rules. 
Under two different sufficient conditions, we show that if $x$ is a \leb{} point, then $\eta(X^x_m)$ converges to $\eta(x)$ in $\cL^1$ (Theorems~\ref{t:ass-proba}~and~\ref{t:ass-nn}).
The first one \eqref{e:ass-proba} is a relaxed measure-continuity condition
: in this case the implication holds regardless of how ties are broken.
The second one \eqref{e:ass-nn} bounds the bias of the tie-breaking rules in relation to the distribution of the $X_k$'s. 
In particular, we present a broad class of tie-breaking rules ---which we call ISIMINs (Independent Selectors of Indices of Minimum Numbers, Definition~\ref{d:ISIMIN})--- for which this condition holds no matter how pathological the distribution of the $X_k$'s is (Proposition~\ref{p:ISIMIN}). 
At a high-level, \aaa{an ISIMIN} selects the smallest among finite sets of numbers, relying on an independent source to break ties.
Notably, both lexicographical and random tie-breaking rules fall within this class.

Furthermore, if neither of the conditions \eqref{e:ass-proba} and \eqref{e:ass-nn} holds, we show with a counterexample (Example~\ref{ex:counter}) that $x$ being a \leb{} point for $\eta$ does \emph{not} imply that $\eta(X^x_m)$ converges to $\eta(x)$ in $\cL^1$, highlighting that tie-breaking rules play a role in the convergence of \nna{}s.

Putting all these results together leads to our characterization of \leb{} points $x$ in terms of the $\cL^1$-convergence of $\eta(X^x_m)$ to $\eta(x)$, which we later extend to general measures (Theorem~\ref{t:general-measures}).

Moreover, the proofs of our main theorems suggest a sequential characterization of \leb{} points, which turns out to be true for (not necessarily bounded) locally-integrable functions (Theorem~\ref{t:charact-gen-measures}).

Notably all of our results on \leb{} points can also be extended to \leb{} values (Section~\ref{s:from-points-to-values}).

We then present some applications. 
For the broad class of tie-breaking rules defined by ISIMINs, we give a detailed proof of the convergence of the risk of the $1$-\nn{} classification algorithm.
This result holds in arbitrary (even non separable) metric spaces where the \leb{}--\besi{} differentiation theorem holds (Theorem~\ref{t:conv-nn-classification}).
This shows in particular (Corollary~\ref{c:const-realizz}) that the consistency of the $1$-\nna{} is essentially equivalent to the realizability assumption \eqref{e:realizability-assumption}.
We conclude the paper with a counterexample (Example~\ref{ex:preiss}) showing that these convergence results do not hold (in general) without assuming the validity of the \leb{}--\besi{} differentiation theorem.

\paragraph{Outline of the paper.}
In Section~\ref{s:leb-vs-nn} we study the relationships between $x$ being a \leb{} point and $\eta(X^x_m)$ converging to $\eta(x)$ in $\cL^1$ in a probabilistic setting, proving our \nn{} characterization of \leb{} points using ISIMINs.
In Section~\ref{s:gen} we extend our results to arbitrary measures, also obtaining a sequential characterization of \leb{} points.
Section~\ref{s:from-points-to-values} illustrates how our findings can be extended to \leb{} values.
In Section~\ref{s:classification} we show an application to the convergence of the risk of binary classification with $1$-\nna{}s defined by ISIMINs.

\section{\leb{} vs \nn}
\label{s:leb-vs-nn}

In this section we study ---in a probabilistic setting--- the relationships between the geometric measure-theoretic concept of \leb{} points and the $\cL^1$-convergence of a $1$-\nn{} regression algorithm for pointwise estimation.

\subsection{Preliminaries and Definitions}
\label{s:prelim}

We begin by introducing our setting, notation, and definitions.
\begin{sett}
Fix an arbitrary metric space $(\cX, d)$.
Let $(\Omega, \cF, \P)$ be a probability space and $X, X_1, X_2, \l$ a sequence of $\cX$-valued $\P$-i.i.d.\ random variables.
\end{sett} 
For each measurable space $(\cW,\cF_{\cW})$ and each random variable $W \colon \Omega \to \cW$, we denote the distribution of $W$ with respect to $\P$ by $\P_W \colon \cF_{\cW} \to [0,1]$, $A \mapsto \P(W \in A)$. 
In the sequel, we will use interchangeably the notations $\P_W(A)$ and $\P (W \in A)$, depending on which one is the clearest in the context. 
We will denote expectations with respect to $\P$ by $\E[\cdot]$.

For any $x \in \cX$ and each $r>0$, we denote the open ball $\bcb{ x' \in \cX \mid d(x,x') < r }$ by $B_r(x)$, the closed ball $\bcb{x' \in \cX \mid d(x,x') \le r }$ by $\bar{B}_r(x)$, and the sphere $\bcb{x' \in \cX \mid d(x,x') = r }$ by $S_r(x)$.

We define the \emph{support} of $\P_X$ as the set
\[
    \supp\brb{ \P_X }
:=
    \Bcb{ x\in \cX \mid \forall r>0, \, \P_X\brb{ \bar{B}_r(x) } > 0 } \;.
\]
We are now ready to introduce \leb{} points in our setting.
\begin{definition}[\leb{} point]
\label{d:leb-point}
Let $\eta:\cX \to \R$ be a bounded measurable function.
We say that a point $x\in \supp\brb{ \P_X }$ is a \emph{\leb{} point} (for $\eta$ with respect to $\P_X$) if
\[
    \frac{\E \Bsb{ \I_{\bar{B}_r(x)}(X) \, \babs{\eta(X) - \eta(x)} } }{\P_X \lrb{ \bar{B}_r(x) }} 
\to 
    0 \;, \qquad \text{as } r \to 0^+ \;.
\]
\end{definition}
Note that \leb{} points have a natural probabilistic interpretation.
Indeed, the ratio in the previous definition, which we call \emph{\leb{} ratio}, is simply the expectation of $|\eta(X)-\eta(x)|$ conditioned to $X \in \bar{B}_r(x)$.

The goal is to characterize \leb{} points in terms of nearest neighbors, which we now introduce formally.
\begin{definition}[Nearest neighbor]
\label{n:tr-set-X'm}
For any point $x\in \cX$ and each $m\in \N$, we say that 
a measurable $X^x_m \colon \Omega \to \cX$ is a \emph{nearest neighbor} of $x$ (among $X_1, \l, X_m$) if, for all $\omega \in \Omega$,
\[
    X^x_m(\omega) 
\in 
    \argmin_{x' \in \{X_1(\omega), \l, X_m(\omega)\}} d(x,x') \;.
\]
\end{definition}
To avoid constant repetitions in our statements, we now fix some notation and the corresponding assumptions.

\begin{assumption}
\label{ass:base}
Until the end of Section~\ref{s:main-section}, we will assume the following:
\begin{enumerate}[topsep = 2pt, parsep = 2pt, itemsep = 2pt]
    \item $x\in \cX$ is a point in the support of $\P_X$;
    \item $\eta : \Omega \to \R$ is a bounded measurable function;
    \item for all $m\in \N$, $X^x_m$ is a nearest neighbor among $X_1, \l, X_m$ of $x$.
\end{enumerate}
\end{assumption}
\noindent For the sake of brevity, we denote $\bar{B}_r(x)$, $B_r(x)$, $S_r(x)$ simply by $\bar{B}_r$, $B_r$, $S_r$.

\subsection{\nn{} \texorpdfstring{$\implies$}{implies} \leb{}}
\label{s:main-section}

In this section we show that if $\eta(X^x_m)$ converges to $\eta(x)$ in $\cL^1$, then $x$ is a \leb{} point. 
We begin by giving a high-level overview of the proof of this result.

Note that, by definition of nearest neighbor $X^x_m$, for any measurable set $A\s \cX$,
\[
    \bcb{ X^x_m \in A \cap \bar{B}_r }
\supset
    \bigcup_{k=1}^m \bbrb{ \bcb{ X_k \in A \cap \bar{B}_r } \cap \bigcap_{i=1,i\neq k}^m \bcb{ X_i \notin \bar{B}_r } } \;,
\]
where the key observation is that the union on the right-hand side is disjoint.
Taking probabilities on both sides and integrating, one can show that
\[
    \frac{ \E \bsb{ \I_{\bar{B}_r} (X) \, \labs{ \eta(X) - \eta(x) } } }{ \P_X \brb{ \bar{B}_r } }
\le
    \frac{\E \bsb{ \labs{ \eta(X^x_m) - \eta(x) } }}{m \, \P_X (\bar{B}_r) \, \brb{1 - \P_X\brb{ \bar{B}_r } }^{m-1}}\;.
\]
This suggests a way to control \leb{} ratios with the $\cL^1$-distance between $\eta(X^x_m)$ and $\eta(x)$.
Tuning $m$ in terms of $r$ will lead to the result.

\begin{theorem}
\label{t:nn-then-leb}
If $\E \bsb{\babs{ \eta(X^x_m) - \eta(x) } } \to 0$ as $m \to \infty$, then $x$ is a \leb{} point.
\end{theorem}

\begin{proof}
Note that, up to a rescaling, we can (and do) assume that $\babs{ \eta - \eta(x) } \le 1$.

If $\P(X = x ) > 0$, then, $x$ is a \leb{} point by the dominated convergence theorem (with dominating function $1$) and the monotonicity of the probability.

Assume then that $\P(X = x ) = 0$.
Note that, for all Borel subset $A$ of $(\cX,d)$, for all $r>0$, and each $m \in \N$,
\[
    \{ X^x_m \in A \cap \bar{B}_r \}
\supset
    \bigcup_{k=1}^m \lrb{ \{ X_k \in A \cap \bar{B}_r \} \cap \bigcap_{i=1,i\neq k}^m \{ X_i \notin \bar{B}_r \} } \;,
\]
where all elements in the union are mutually disjoint, then
\begin{align*}
    \P (X^x_m \in A \cap \bar{B}_r)
& \ge
    \P \lrb{ \bigcup_{k=1}^m \lrb{ \{ X_k \in A \cap \bar{B}_r \} \cap \bigcap_{i=1,i\neq k}^m \{ X_i \notin \bar{B}_r \} }}
\\
& =
    \sum_{k=1}^m \P \lrb{ \{ X_k \in A \cap \bar{B}_r \} \cap \bigcap_{i=1,i\neq k}^m \{ X_i \notin \bar{B}_r \} }
\\
& =
    m \, \P_X (A \cap \bar{B}_r) \, \brb{ 1 - \P_X (\bar{B}_r) }^{m-1}
\end{align*}
which in turn gives
\[
    \E \Bsb{ \I_{\bar{B}_r} (X^x_m) \, \babs{ \eta(X^x_m) - \eta(x) } }
\ge
    m \, \E \Bsb{ \I_{\bar{B}_r} (X) \, \babs{ \eta(X) - \eta(x) } } \, \brb{1 - \P_X(\bar{B}_r) }^{m-1}\;,
\]
which rearranging, upper bounding $\I_{\bar{B}_r} (X^x_m) $ with $1$, and dividing both sides by $\P_X (\bar{B}_r)$, yields
\begin{equation}
    \label{e:vs}
    \frac{ \E \Bsb{ \I_{\bar{B}_r} (X) \, \babs{ \eta(X) - \eta(x) } } }{ \P_X (\bar{B}_r)}
\le
    \frac{\E \Bsb{ \babs{ \eta(X^x_m) - \eta(x) } }}{m \, \P_X (\bar{B}_r) \, \brb{1 - \P_X(\bar{B}_r) }^{m-1}}\;,
\end{equation}
if $\P_X(\bar{B}_r)< 1$.
We will show that the right hand side vanishes as $m$ approaches $\iop$.
Fix any $\e > 0$. By assumption there exists $M \in \N$ such that, for all $m \in \N$, $m \ge M$, 
\begin{equation}
    \label{e:vs-2}
    \E \Bsb{ \babs{ \eta(X^x_m) - \eta(x) } } \le \frac{\e}{e} \;.
\end{equation}
For each $m \in \N$, define the (smooth) auxiliary function
\begin{align*}
    f_m \colon [0,1] & \to [0,\iop)\;,\\
    t & \mapsto m \, t \, (1-t)^{m-1}\;.
\end{align*}
By studying the sign of its first derivative, we conclude \aaa{that}, for each $m\in \N$, $f_m$ is increasing on $[0, 1/m]$, it is decreasing on $[1/m, 1]$, and 
\[
    \max_{t \in [0,1]} f_m(t)
=
    f_m \lrb{ \frac{1}{m} }
=
    \lrb{ 1-\frac{1}{m} } ^{m-1}
> 
    \frac{1}{e}\;.
\]
Hence, the superlevel set $\{f_m \ge 1/e\}$ is a non-empty and closed subinterval of $[0,1]$. 
For all $m\in \N$, we let 
\begin{equation}
    \label{e:vs-3}
    I_m:= [a_m, b_m] := \{f_m \ge 1/e\} \;.
\end{equation}
Note that for all $m\in \N$, $b_{m+1} \in I_m$, i.e., that $f_m(b_{m+1}) \ge 1/e$. 
Indeed, since for all $m\in \N$, $f_{m+1}(b_{m+1}) = 1/e$ and $b_{m+1} \ge 1/(m+1)$, we have that
\begin{align*}
    f_m(b_{m+1}) - \frac{1}{e}
&=
    f_m(b_{m+1}) - f_{m+1} (b_{m+1})
\\
&=
    m \, b_{m+1} \, ( 1 - b_{m+1})^{m-1}  - (m+1) \, b_{m+1} \, ( 1 - b_{m+1})^{m}
\\
&=
    b_{m+1}(1-b_{m+1})^{m-1}\brb{(m+1) \,b_{m+1}-1}
\\
&\ge
    b_{m+1}(1-b_{m+1})^{m-1}\lrb{(m+1) \,\frac{1}{m+1} -1}
= 0 \;.
\end{align*}
\aaa{This} implies that for all $m\in \N$, $a_m \le b_{m+1} \le b_m$, thus $I_{m} \cap I_{m+1} \neq \varnothing$ and $(b_m)_{m\in \N}$ is non-increasing.
Finally, $a_m \to 0$ as $m \to \iop$, since for all $m\in \N$, $a_m \le 1/m$.
Hence
\[
    \bigcup_{m \in \N, m \ge M} I_m = (0,b_M] \;.
\]
Let $\delta > 0$ such that for each $r \in (0,\delta)$ we have that \aaa{$\P_X(\bar{B}_r) \in (0, b_M] $.}
Thus, for each $r \in (0,\delta)$, there exists $m \in \N$ such that $m \ge M$ and $\P_X(\bar{B}_r) \in I_m$, yielding
\begin{multline*}
    \frac{ \E \Bsb{ \I_{\bar{B}_r} (X) \, \babs{ \eta(X) - \eta(x) } } }{ \P_X (\bar{B}_r)}
\overset{\eqref{e:vs}}{\le}
    \frac{\E \Bsb{ \babs{ \eta(X^x_m) - \eta(x) } }}{m \, \P_X (\bar{B}_r) \, \brb{1 - \P_X(B_r) }^{m-1}}
\\
=
    \frac{\E \Bsb{ \babs{ \eta(X^x_m) - \eta(x) } }}{f_m(\P_X(\bar{B}_r))}
\overset{\eqref{e:vs-3}}{\le}
    e \, \E \Bsb{ \babs{ \eta(X^x_m) - \eta(x) } }
\overset{\eqref{e:vs-2}}{\le}
    \e\;.
\end{multline*}
Being $\e$ arbitrary, we conclude that $x$ is a \leb{} point.
\end{proof}

\subsection{\leb{} \texorpdfstring{$\implies$}{implies} \nn{} (sometimes)}
\label{s:main-section-part-2}

In this section we will assume that $x$ is a \leb{} point and study when this implies that $\eta(X^x_m)$ converges to $\eta(x)$ in $\cL^1$. 

We begin by addressing two trivial cases. 
The first one is when $x$ is an atom for $\P_X$. 
In this case, the result is trivialized by the fact that $X^x_m$ becomes eventually \emph{equal} to $x$ (almost surely).
The second one is when $\E \bsb{ \I_{\bar{B}_r}(X) \, \babs{\eta(X) - \eta(x)} } = 0$ for some $r>0$.
In this case, since $x$ belongs to the support of $\P_X$, the result is trivialized by the fact that $X^x_m$ will eventually fall inside $\bar{B}_r$ (almost surely).
These ideas are made rigorous in the following lemma.
\begin{lemma}
\label{l:trivia-cases}
If one of the two following conditions is satisfied:
\begin{enumerate} [topsep = 2pt, parsep = 2pt, itemsep = 2pt]
    \item \label{i:trivial-1} $\P ( X = x )>0$;
    \item \label{i:trivial-2} there exists $r>0$ such that $\E\bsb{ \I_{\bar{B}_r}(X) \, \babs{\eta(X) - \eta(x)} } = 0$;
\end{enumerate}
 then $\E \bsb{ \babs{ \eta(X^x_m) - \eta(x) } } \to 0$ as $m\to \iop$.
\end{lemma}
\begin{proof}
If condition~\ref{i:trivial-1} is satisfied, we have that
\begin{align*}
    \E \Bsb{\babs{ \eta(X^x_m) - \eta(x) } } 
& = 
    \E \Bsb{ \I_{ \cX \m \{x\} } \brb{ X^x_m } \, \babs{ \eta(X^x_m) - \eta(x) } }  
\le
    2 \, \lno{ \eta }_\iop \, \P(X^x_m \neq x)
\\
&=
    2 \, \lno{ \eta }_\iop \, \P \lrb{ \bigcap_{k = 1}^m \{ X_k \neq x \} }
=
    2 \, \lno{ \eta }_\iop \, \prod_{k = 1}^m \P \lrb{ X_k \neq x }
\\
&=
    2 \, \lno{ \eta }_\iop \, \brb{ 1 - \P(X = x)}^m \to 0 \;,
    \qquad \text{ as } m \to \infty \;.
\end{align*}
Assume now that condition~\ref{i:trivial-2} is satisfied.
Then, being $\P_{X^x_m}$ absolutely continuous with respect to $\P_X$, it follows that $\E \bsb{ \I_{\bar{B}_r}(X^x_m) \, \babs{\eta(X^x_m) - \eta(x)} } = 0$. Since $x \in \supp(\P_X)$, we have that $\P_X(\bar{B}_r ) > 0$, which in turn gives
\begin{align*}
    \E \Bsb{\babs{ \eta(X^x_m) - \eta(x) } } 
& = 
    \E \Bsb{\I_{\bar{B}_r^c}(X^x_m)\babs{ \eta(X^x_m) - \eta(x) } } 
\le 
    2 \, \lno{ \eta }_\iop \, \P \brb{ X^x_m \notin \bar{B}_r }
\\ 
&= 
    2 \, \lno{ \eta }_\iop \, \P \lrb{ \bigcap_{k = 1}^m \lcb{ X_k \notin \bar{B}_r } } 
= 
    2 \, \lno{ \eta }_\iop \, \prod_{k = 1}^m \P \lrb{ X_k \notin \bar{B}_r } 
\\
&= 
    2 \, \lno{ \eta }_\iop \, \brb{ 1 - \P_X(\bar{B}_r )}^m \to 0 \;,
    \qquad \text{ as } m \to \infty \;. \qedhere
\end{align*}
\end{proof}
By the previous lemma, without loss of generality, we can (and do) assume that none of the two previous conditions hold.
\begin{assumption}
\label{a:non-trivial}
Until the end of this section, we will assume the following:
\begin{enumerate} [topsep = 2pt, parsep = 2pt, itemsep = 2pt]
    \item \label{i:non-trivial-1} $\P ( X = x )=0$;
    \item \label{i:non-trivial-2} $\E \bsb{ \I_{\bar{B}_r}(X) \, \babs{\eta(X) - \eta(x)} } > 0$, for all $r>0$.
\end{enumerate}
\end{assumption}

We now proceed to estimate the expectation $\E \bsb{\babs{ \eta(X^x_m) - \eta(x) } }$.
Due to the nature of nearest neighbors, one could figure that $\E \bsb{\babs{ \eta(X^x_m) - \eta(x) } }$ behaves differently if $X^x_m$ is close by, or far away from $x$.
This idea leads to the splitting of the expectation $\E \bsb{\babs{ \eta(X^x_m) - \eta(x) } }$ on the region in which $X^x_m$ belongs to a closed ball $\bar{B}_r$ and its complement, i.e. on $\{X^x_m \in \bar{B}_r\}$ and $\{X^x_m \in \bar{B}_r^c\}$. Since ties might raise issues on spheres, we further split $\{X^x_m \in \bar{B}_r\}$ into $\{X^x_m \in B_r\}$ and $\{X^x_m \in S_r\}$.

The next lemma gives estimates of the three terms determined by this splitting.

\begin{lemma}
\label{l:geom-idea}
For all $r>0$ and all $m\in \N$,
\begin{align}
    \label{e:term-interior}
    \E \Bsb{\I_{ B_r }(X^x_m) \babs{ \eta(X^x_m) - \eta(x) } }
& \le
    m \, \E \Bsb{\I_{ B_r }(X) \babs{ \eta(X) - \eta(x) } } \;,
\\
    \label{e:term-sphere}
    \E \Bsb{\I_{ S_r }(X^x_m) \babs{ \eta(X^x_m) - \eta(x) } }
&\le
    2 \,\lno{ \eta }_\infty \, m \, \P_X \lrb{ S_{r} } \, \exp \brb{ -(m-1) \, \P_X \lrb{ B_{r} } } \;,
\\
    \label{e:term-exterior}
    \E \Bsb{\I_{ B_r^c }(X^x_m) \babs{ \eta(X^x_m) - \eta(x) } }
& \le
    2 \, \lno{ \eta }_\infty \, \exp \brb{ -m \, \P_X(\bar{B}_r) } \aaa{\;.}
\end{align}
\end{lemma}
\begin{proof}
Fix any $r>0$ and $m\in \N$.

We begin by proving inequality \eqref{e:term-interior}.
For all Borel sets $A$ of $(\cX, d)$, if $X^x_m \in A$, then at least one of the $X_i$'s belongs to $A$, i.e.,
\[ 
    \{ X_i \in A \} 
\s
    \bigcup_{k=1}^m \{ X_k \in A \} \;.
\]
This yields
\[
    \P (X^x_m \in A)
\le
    \P\lrb{ \bigcup_{k=1}^m \{ X_k \in A \} }
\le
    \sum_{k=1}^m \P ( X_k \in A )
=
    m \P(X \in A) \;,
\]
hence $\P_{X^x_m} \le m \, \P_X$, which in turn gives, for any measurable function $f\colon \cX \to [0,\iop]$, that $\E \bsb{ f(X^x_m) } \le m \E \bsb{ f(X) } $. 
Then
\[
    \E \Bsb{\I_{ B_r }(X^x_m) \, \babs{ \eta(X^x_m) - \eta(x) } }
\le
    m \, \Bsb{\I_{ B_r }(X) \, \babs{ \eta(X) - \eta(x) } }\;.
\]
This proves \eqref{e:term-interior}.

We now prove inequality \eqref{e:term-sphere}. Note that if $X^x_m \in S_r$ then at least one of the $X_k$'s belongs to $S_r$, while the others can't fall in $B_r$ (and vice versa), i.e.,
\[
   \lcb{  X^x_m \in S_r }
    \aaa{=}
    \bigcup_{k=1}^m\lrb{ \bcb{ X_k \in S_r} \cap \bigcap_{i=1, i \neq k}^{m} \bcb{ X_i \notin B_r } } \;.
\]
This yields
\begin{multline*}
    \E \Bsb{\I_{S_r} \brb{X^x_m} \, \babs{ \eta(X^x_m) - \eta(x) } } 
\le 
    2 \,\lno{ \eta }_\infty \, \P \brb{ X^x_m \in S_r }
\\
\begin{aligned}
&\le 
    2 \,\lno{ \eta }_\infty \, \P \lrb{ \bigcup_{k=1}^{m} \lrb{ \bcb{ X_k \in S_r} \cap \bigcap_{i=1, i \neq k}^{m} \bcb{ X_i \notin B_r } } } 
\\
&\le 
    2 \,\lno{ \eta }_\infty \, \sum_{k=1}^{m} \lrb{ \P \lrb{ X_k \in S_r } \prod_{i=1, i \neq k}^{m} \P \lrb{ X_i \notin B_r } } 
\\
&= 
    2 \,\lno{ \eta }_\infty \, m \, \P_X \lrb{ S_r } \, \brb{ 1 - \P_X \lrb{ B_r } }^{m-1}  
\\ 
&\le 
    2 \,\lno{ \eta }_\infty \, m \, \P_X \lrb{ S_r } \, \exp \brb{ -(m-1) \, \P_X \lrb{ B_r } } \;.
\end{aligned}
\end{multline*}
This proves \eqref{e:term-sphere}.

Finally, we prove inequality \eqref{e:term-exterior}.
Note that $X^x_m \notin \bar{B}_r$ is equivalent to the fact that none of the $X_k$'s belong to $\bar{B}_r$, i.e.,
\[
    \{ X^x_m \notin \bar{B}_r \}
= 
    \bigcap_{k = 1}^m \lcb{ X_k \notin \bar{B}_r } \;,
\]
then
\begin{multline*}
    \E \Bsb{\I_{\bar{B}_r^c} (X^x_m) \, \babs{ \eta(X^x_m) - \eta(x) } }
\le
    2 \,\lno{ \eta }_\infty \, \P \brb{ X^x_m \notin \bar{B}_r }
\\
\begin{aligned}
&= 
    2 \,\lno{ \eta }_\infty \, \P \lrb{ \bigcap_{k = 1}^m \lcb{ X_k \notin \bar{B}_r } } 
= 
    2 \,\lno{ \eta }_\infty \, \prod_{k = 1}^m \P \lrb{ X_k \notin \bar{B}_r } 
\\ &
= 
    2 \,\lno{ \eta }_\infty \, \brb{ 1 - \P_X(\bar{B}_r )}^m 
\le 
    2 \,\lno{ \eta }_\infty \, \exp \brb{ -m \, \P_X(\bar{B}_r) } \;.
\end{aligned}
\end{multline*}
This concludes the proof.
\end{proof}

We now give a high-level overview of the ideas used to prove the $\cL^1$-convergence of $\eta(X^x_m)$ to $\eta(x)$. 
For the sake of simplicity, assume for now that the cumulative function of $d(X,x)$ is continuous around $0$. 
This measure continuity-condition is equivalent to:
\begin{equation}
\label{e:cont-condition}
\exists R >0, \, \forall r\in (0, R), \ \P_X(S_r) = 0 \;.
\end{equation}
In this case, the upper bound in \eqref{e:term-sphere} is always $0$. Therefore,  Lemma~\ref{l:geom-idea} implies, for each $m\in \N$ and all \aaa{$r\in(0,R)$}, that
\begin{equation}
    \label{e:estimate-intuition}
    \E \Bsb{ \babs{ \eta(X^x_m) - \eta (x) } }
\le
    m \, \E \Bsb{\I_{ B_r }(X) \babs{ \eta(X) - \eta(x) } }
    +
    2 \, \lno{ \eta }_\infty \, \exp \brb{ -m \, \P_X(\bar{B}_r) } \;.
\end{equation}
This bound might seem pointless (under \aaa{Assumption}~\ref{a:non-trivial}) since for all \emph{fixed} $r>0$, the first term on the right hand side diverges as $m$ approaches infinity.
The idea is then to pick \aaa{a sequence} of radii $r=r_m$ that vanishes in a way that this first term goes to zero. 
This is easily achievable, e.g. by selecting $(r_m)_{m \in \N}$ that \aaa{decreases} to zero very quickly. 
However, in this case it is now the second term that may not vanish, since $m \P_X(\bar{B}_{r_m})$ may not \aaa{diverge} to $\infty$. 
The key is then to find a trade-off between the two competing terms.
This would be achieved if one could pick a sequence $(r_m)_{m\in \N}$ so that
\begin{equation}
    \label{e:tuning-r}
    m
    \,
    \sqrt{ \E \bsb{ \I_{\bar{B}_{r_m}(x)}(X) \, \labs{\eta(X) - \eta(x)} } } 
    \, 
    \sqrt{ \P_X \brb{ \bar{B}_{r_m}} }
=
    1 \;.
\end{equation}
Indeed, in this case, inequality \eqref{e:estimate-intuition} together with the measure-continuity condition \eqref{e:cont-condition} yields
\begin{align*}
    \E \Bsb{ \babs{ \eta(X^x_m) - \eta (x) } }
&
\le
    \lrb{ \frac{\E \bsb{\I_{ \bar{B}_{r_m} }(X) \babs{ \eta(X) - \eta(x) } }} 
    {\P_X \brb{ \bar{B}_{r_m}}} }^{1/2}
\\
&
    \quad +
    2 \, \lno{ \eta }_\infty \, \exp \lrb{ -\lrb{ \frac{\E \bsb{\I_{ \bar{B}_{r_m} }(X) \babs{ \eta(X) - \eta(x) } }} 
    {\P_X \brb{ \bar{B}_{r_m}}} }^{-1/2} }
\end{align*}
which vanishes if $x$ is a \leb{} point.

Under the current assumptions, one can indeed show the existence of a sequence $(r_m)_{m\in \N}$ satisfying \eqref{e:tuning-r}.
However, if the measure-continuity condition \eqref{e:cont-condition} does not hold, this might no longer be the case.
In order to address more general cases, we introduce the following definition.
\begin{definition}[$\alpha$-sequence]
\label{d:alpha-seq-1}
Fix any $\alpha \in (0,1)$. For all $r>0$, define
\[
    M_\alpha(r)
:=
    \bbrb{ \E \Bsb{\aaa{\I_{\bar{B}_{r}}(X)} \, \babs{ \eta(X) - \eta(x) } }  }^{\alpha} 
    \Brb{ \P_X \brb{ \bar{B}_r} }^{1-\alpha} \;.
\]
Define $m_1 := \bce{ 1/ M_\alpha(1) }$. 
For all $m \in \N$ such that $m \ge m_1$, define
\[
    r_m
:=
    \sup \lcb{ r > 0 \mid M_\alpha(r) < \frac{1}{m} } \;.
\]
We say that $(r_m)_{m\in \N, m \ge m_1}$ is the \emph{$\alpha$-sequence} (for $\eta$ with respect to $\P_X$) at $x$.
\end{definition}
The following lemma states several useful properties of $\alpha$-sequences.
\begin{lemma}
\label{l:incistato}
Fix any $\alpha\in (0,1)$. Then, the $\alpha$-sequence $(r_m)_{m\in \N, m \ge m_1}$ is a well-defined vanishing sequence of strictly positive numbers.
Moreover, for all $m\in \N$ with $m\ge m_1$, we have
\begin{align}
    \label{e:m-2}
    m & \le 
    \frac{1}{\E \Bsb{\I_{ B_{r_m} }(X) \babs{ \eta(X) - \eta(x) } }}
    \lrb{ 
        \frac{
            \E \Bsb{ \I_{B_{r_m}}(X) \, \babs{\eta(X) - \eta(x)} }
        }
        {
            \P_X \brb{ B_{r_m}}
        } 
    }^{1-\alpha}  
\\
    \label{e:m-1}
    m & \ge
    \frac{1}{\P_X \lrb{ \bar{B}_{r_m}}}
    \lrb{ 
        \frac
        {
            \E \bsb{ \I_{\bar{B}_{r_m}}(X) \, \babs{\eta(X) - \eta(x)} }  
        }
        {
            \P_X \lrb{ \bar{B}_{r_m}}
        }
    }^{-\alpha} 
\end{align}
where we stress that the balls $B_{r_m}$ in \eqref{e:m-2} are open.
\end{lemma}
\begin{proof}
Note that the function $r\mapsto M_\alpha(r)$ is non-decreasing, right-continuous (by the continuity from above of finite measures), and for each $R>0$, it satisfies
\begin{equation}
    \label{e:M(r)}
    M_\alpha(r) 
\uparrow 
    \bbrb{ 
        \E \Bsb{ \I_{B_R}(X) \, \babs{\eta(X) - \eta(x)} }  
    }^{\alpha} 
    \Brb{ 
        \P_X \brb{ B_R} 
    }^{1-\alpha} \;, 
    \quad \text{as } r \uparrow R
\end{equation}
(by the continuity from below of measures), where we stress that the balls $B_R$ in the previous formula are open. Furthermore, by Assumption~\ref{a:non-trivial} and the fact that $x$ belongs to the support of $\P_X$, we have that $0<M_\alpha(r) \downarrow 0$ as $r \downarrow 0$.
This implies that the $\alpha$-sequence $(r_m)_{m\in \N, \, m \ge m_1}$ is well-defined, strictly positive, and $r_m \downarrow 0$ as $m \uparrow \iop$.

Thus, for each $m \in \N$ such that $m \ge m_1$, if $r \in (0, r_m)$, then $M_\alpha(r) < 1/m$, which in turn yields
\[
    \lim_{r \uparrow r_m} M_\alpha(r) \le 1/m \;.
\]
Using \eqref{e:M(r)} and rearranging gives \eqref{e:m-2}.

Analogously, for each $m \in \N$ such that $m \ge m_1$, if $r> r_m$, then $M(r) \ge 1/m$, which in turn yields
\[
    M(r_m) = \lim_{r \downarrow r_m} M(r) \ge 1/m \;.
\]
Rearranging gives \eqref{e:m-1} and completes the proof.
\end{proof}
Before proceeding with the main results of the section, we need one simple (geometric measure theory flavored) lemma.
\begin{lemma}
\label{l:geom-measure-flavor}
If $x$ is a \leb{} point and $(r_m)_{m \in \N}$ is a vanishing sequence of strictly positive real numbers, then
\begin{equation}
    \label{e:geom-measure-flavor}
    \frac
    {
        \E \Bsb{ \I_{B_{r_m}}(X) \, \babs{\eta(X) - \eta(x)} }
    }
    {
        \P_X (B_{r_m})
    }
    \to 0 \;,
    \qquad m \to \iop
\end{equation}
where we stress that the balls $B_{r_m}$ in the previous formula are open.
\end{lemma}
\begin{proof}
By the continuity from below of measures, we \aaa{have that for} all $m \in \N$, there exists $\rho_m > 0$ such that $\labs{r_m - \rho_m} \le 1/m$ and
\begin{equation*}
    \labs{
        \frac{\E \Bsb{ \I_{B_{r_m}}(X) \, \babs{\eta(X) - \eta(x)} }}
            {\P_X \brb{ B_{r_m}}}
        - 
        \frac{\E \Bsb{ \I_{\bar{B}_{\rho_m}}(X) \, \babs{\eta(X) - \eta(x)} }}
            {\P_X \brb{ \bar{B}_{\rho_m}}} 
    } 
\le 
    \frac{1}{m} \;.
\end{equation*}
Since $r_m \to 0^+$ as $m \to \infty$, we have that also $\rho_m \to 0^+$, as $m \to \infty$. Being $x$ a \leb{} point, it follows that
\begin{equation*}
\frac{\E \Bsb{ \I_{\bar{B}_{\rho_m}}(X) \, \babs{\eta(X) - \eta(x)} }}{\P_X \brb{ \bar{B}_{\rho_m}}} \to 0 \;, \qquad \text{as } m \to \infty \;,
\end{equation*}
which in turn implies, for $m \to \iop$,
\begin{equation*}
\frac{\E \Bsb{ \I_{B_{r_m}}(X) \, \babs{\eta(X) - \eta(x)} }}{\P_X \brb{ B_{r_m}}} \le \frac{1}{m} + \frac{\E \Bsb{ \I_{\bar{B}_{\rho_m}}(X) \, \babs{\eta(X) - \eta(x)} }}{\P_X \brb{ \bar{B}_{\rho_m}}} \to 0 \;.
\end{equation*}
\end{proof}
The following result showcases the usefulness of $\alpha$-sequences.
\begin{theorem}
\label{t:victory-spheres}
Let $(r_m)_{m\in \N, m \ge m_1}$ be an $\alpha$-sequence for some $\alpha \in (0,1)$. 
If $x$ is a \leb{} point, then the following are equivalent:
\begin{enumerate} [topsep = 2pt, parsep = 2pt, itemsep = 2pt]
\item \label{i:alpha-paua-1} $\E \bsb{\babs{ \eta(X^x_m) - \eta(x) } } \to 0$, as $m\to \iop$;
\item \label{i:alpha-paua-2} $\E \bsb{\I_{ S_{r_m} }(X^x_m) \, \babs{ \eta(X^x_m) - \eta(x) } } \to 0$, as $m\to \iop$.
\end{enumerate}
\end{theorem}
\begin{proof}
We prove the non-trivial implication \ref{i:alpha-paua-2}$~\Rightarrow$~\ref{i:alpha-paua-1}.
Recall that $r_m \to 0^+$ as $m \to \infty$ by Lemma~\ref{l:incistato}.
For each $m \in \N$, if $m \ge m_1$, we have
\begin{align*}
    \E \bsb{\babs{ \eta(X^x_m) - \eta(x) } }
& =
    \E \bsb{\I_{B_{r_m}}(X^x_m) \, \babs{ \eta(X^x_m) - \eta(x) } }
\\& \qquad
    +
    \E \bsb{\I_{S_{r_m}}(X^x_m) \, \babs{ \eta(X^x_m) - \eta(x) } }
\\& \qquad \qquad
    +
    \E \bsb{\I_{\bar{B}_{r_m}^c}(X^x_m) \, \babs{ \eta(X^x_m) - \eta(x) } }
\\&
    =: 
    (\mathrm{I}) + (\mathrm{II}) + (\mathrm{III})\;.
\end{align*}
We will show that all three terms above approach $0$ as $m\to \iop$.

To show that the first vanishes, we apply inequality \eqref{e:term-interior} in Lemma~\ref{l:geom-idea}, inequality \eqref{e:m-2} in Lemma~\ref{l:incistato}, and Lemma~\ref{l:geom-measure-flavor}:
\begin{align*}
    (\mathrm{I}) 
& \overset{\phantom{\eqref{e:m-2}}}{=} 
    \E \Bsb{\I_{ B_{\aaa{r_m}} }(X^x_m) \,\babs{ \eta(X^x_m) - \eta(x) } }
\overset{\eqref{e:term-interior}}{\le}
    m \, \E \Bsb{\I_{ B_{\aaa{r_m}} }(X) \,\babs{ \eta(X) - \eta(x) } }
\\
& \overset{\eqref{e:m-2}}{\le}
    \lrb{ 
        \frac{
            \E \Bsb{ \I_{B_{r_m}}(X) \, \babs{\eta(X) - \eta(x)} }
        }
        {
            \P_X \brb{ B_{r_m}}
        } 
    }^{1-\alpha}
    \overset{\eqref{e:geom-measure-flavor}}{\longrightarrow} 0 \;,
    \qquad \text{as } m \to \iop\;.
\end{align*}

The second term vanishes by assumption.

To show that the third term vanishes, we apply inequality \eqref{e:term-exterior} in Lemma~\ref{l:geom-idea} and inequality \eqref{e:m-1} in Lemma~\ref{l:incistato}:
\begin{align*}
    (\mathrm{III}) 
&\overset{\phantom{\eqref{e:m-2}}}{=} 
    \E \Bsb{\I_{ \bar{B}_{\aaa{r_m}} ^c }(X^x_m) \, \babs{ \eta(X^x_m) - \eta(x) } }
\overset{\eqref{e:term-exterior}}{\le} 
    2 \, \lno{ \eta }_\infty \, \exp \brb{ -m \, \P_X(\bar{B}_{\aaa{r_m}}) }
\\
&\overset{\eqref{e:m-1}}{\le}
    2 \, \lno{ \eta }_\infty \, \exp \lrb{ - \lrb{ 
        \frac
        {
            \E \bsb{ \I_{\bar{B}_{r_m}}(X) \, \babs{\eta(X) - \eta(x)} }  
        }
        {
            \P_X \lrb{ \bar{B}_{r_m}}
        }
    }^{-\alpha}  } \to 0 \;,
\end{align*}
as $m \to \iop$.
\end{proof}
Thanks to the previous theorem, to prove the $\cL^1$-convergence of $\eta(X^x_m)$ to $\eta(x)$, we only need to control the expectation on spheres along an $\alpha$-sequence.
Recall the measure-continuity condition \eqref{e:cont-condition}.
By the previous theorem, under this condition, if $x$ is a \leb{} point, then automatically $\eta(X^x_m)$ converges to $\eta(x)$ in $\cL^1$.
Theorem~\ref{t:ass-proba} will show that the same thing holds under the following weaker condition:
\begin{equation}
    \label{e:ass-proba}
    \exists K \ge 0, \, \exists R>0, \, \forall r \in (0,R), \ \P_X(S_r) \le K \, \P_X(B_r) \;.
\end{equation}
To give some intuition on condition \eqref{e:ass-proba}, consider the cumulative $F$ of the random variable $d(x,X)$ that evaluates the distance between $x$ and $X$.
For all $r > 0$, we have $F(r) = \P\brb{ d(x,X) \le r } = \P_X (\bar{B}_r)$.
Then, condition \eqref{e:ass-proba} can be restated as
\[
    \exists K \ge 0, \, \exists R>0, \, \forall r \in (0,R), \  \frac{F(r)}{F(r^-)} \le K + 1 \;,
\]
where $F(r^-) := \lim_{\rho \to r^-} F(\rho) = \P_X(B_r)$.
This is a relaxation on the continuity of $F$ in a neighbor of $0$ ---which is precisely the case $K=0$, corresponding to the measure-continuity condition \eqref{e:cont-condition}--- and it allows infinite discontinuity points around $0$.

In order to prove Theorem~\ref{t:ass-proba} as well as the last theorem of the section, we will need the following technical lemma.
\begin{lemma}
\label{l:ass-proba-tech}
Let $(m_j)_{j\in \N}$ be a strictly monotone sequence of natural numbers and $(\rho_j)_{j\in\N}$ a sequence of strictly positive real numbers.
If both of the following conditions hold:
\begin{enumerate} [topsep = 2pt, parsep = 2pt, itemsep = 2pt]
\item\label{i:c1} there exists $K\ge 0$ such that, for all $j \in \N$, $\P_X\brb{S_{\rho_j}} \le K \, \P_X\brb{B_{\rho_j}}$;
\item \label{i:c2} $m_j \P_X \brb{ \bar{B}_{\rho_j} } \to \iop$ as $j\to \iop$;
\end{enumerate}
then
\[
    \E \Bsb{ \I_{S_{\rho_j}} \brb{ X^x_{m_j} } \, \babs{ \eta \brb{ X^x_{m_j} } - \eta(x) } } \to 0 \;, \qquad \text{as } j \to \iop \;.
\]
\end{lemma}
\begin{proof}
We show first that condition \ref{i:c2} still holds if closed balls $\bar{B}_{\rho_j}$ are replaced by open balls $B_{\rho_j}$.
Using conditions \ref{i:c1} and \ref{i:c2} yields the strengthened convergence condition:
\[
    m_j \P_X\brb{ B_{\rho_j} }
\ge
    \frac{1}{K+1} \Brb{ m_j \P_X \brb{S_{\rho_j}} + m_j \P_X \brb{B_{\rho_j}} }
=
    \frac{1}{K+1} m_j \P_X \brb{\bar{B}_{\rho_j}}
\to \iop \;,
\]
as $j \to \iop$. 
Hence, recalling inequality~\eqref{e:term-sphere}, using condition~\ref{i:c1}, and applying the strengthened convergence condition, we have
\begin{multline*}
    \E \Bsb{ \I_{S_{\rho_j}} \brb{ X^x_{m_j} } \, \babs{ \eta \brb{ X^x_{m_j} } - \eta(x) } }
\le
    2 \lno{ \eta }_\iop m_j \P_X\brb{ S_{\rho_j} } \exp \Brb{ -(m_j-1) \P_X \brb{ B_{\rho_j} } }
\\
\le
    2 K \lno{ \eta }_\iop m_j \P_X\brb{ B_{\rho_j} } \exp \Brb{ -(m_j-1) \P_X \brb{ B_{\rho_j} } }
    \to 0\;, \qquad \text{as } j \to \iop \;.
    \tag*{\qedhere}
\end{multline*}
\end{proof}
We can now prove one of our main results which shows that, under the weakened measure-continuity assumption~\eqref{e:ass-proba}, if $x$ is a \leb{} point, then $\eta(X^x_m)$ converges to $x$ in $\cL^1$.
\begin{theorem}
\label{t:ass-proba}
If condition~\eqref{e:ass-proba} is satisfied and $x$ is a \leb{} point, then 
\[
    \E \bsb{ \babs{ \eta(X^x_m) - \eta(x) } } \to 0\;, \qquad \text{as } m \to \iop \;.
\] 
\end{theorem}
\begin{proof}
Take any $\alpha$-sequence $(r_m)_{m\in \N, m\ge m_1}$, for some $\alpha\in (0,1)$.
By Theorem~\ref{t:victory-spheres}, it suffices to prove that $\E \bsb{\I_{ S_{r_m} }(X^x_m) \, \babs{ \eta(X^x_m) - \eta(x) } } \to 0$, as $m\to \iop$. 
Recall that $0 < r_m \to 0$ as $m\to \iop$ (Lemma~\ref{l:incistato}).
By assumption~\eqref{e:ass-proba} and Lemma~\ref{l:ass-proba-tech} it is then sufficient to prove that $m \P_X \brb{ \bar{B}_{r_m} } \to \iop$ as $m \to \iop$.
Being $x$ a \leb{} point, this follows immediately by \eqref{e:m-1}.
\end{proof}
The previous result shows that there is a large class of distributions $\P_X$ for which $\eta(X^x_m)$ converges to $\eta(x)$ in $\cL^1$, no matter how ties are broken in the definition of the nearest neighbor $X^x_m$.

Vice versa, we will show there exists a large class of tie-breaking rules for which $\eta(X^x_m)$ converges to $\eta(x)$ in $\cL^1$, no matter how pathological $\P_X$ is. This is a consequence of Theorem~\ref{t:ass-nn}, together with the forthcoming Proposition~\ref{p:ISIMIN}. 
The theorem relies on the following condition,  whose purpose \aaa{is to bound} the bias of the tie-breaking rule \aaa{(for a concrete example, see Example~\ref{ex:counter-rev})}:
\begin{multline}
    \label{e:ass-nn}
    \exists C>0, \, \exists R > 0, \, \exists M > 0, \, \forall r \in (0,R), \, \forall m \in \{M, M+1, \l \}, \, \brb{ \P_{X} (S_r) > 0 }
\\
\implies 
    \E \bsb{ \labs{ \eta(X^x_m) - \eta(x) } \mid X^x_m \in S_r }
\le
    C \, \E \bsb{ \labs{ \eta(X) - \eta(x) } \mid X \in S_r } \;.
\end{multline}
Note that the previous expression is well-defined since the condition $\P_{X} (S_r) > 0$ is equivalent to $\P_{X^x_m} (S_r) > 0 $, as shown in the following lemma.
\begin{lemma}
\label{l:abs-cont-X-Xm}
Let $r>0$. Then the following are equivalent:
\begin{enumerate} [topsep = 2pt, parsep = 2pt, itemsep = 2pt]
\item $\P_X(S_r) >0$;
\item there exists $m\in \N$ such that $\P_{X^x_m} (S_r) >0$;
\item for all $m\in \N$, we have that $\P_{X^x_m} (S_r) >0$.
\end{enumerate}
\end{lemma}
\begin{proof}
The results follow by the fact that, for all $m\in \N$, we have 
\[
    \bigcup_{k=1}^m \{ X_k \in S_r \}
\supset
    \{ X^x_m \in S_r \}
\supset
    \bigcap_{k=1}^m \{ X_k \in S_r \} \;,
\]
which in turn implies
\[
    m \,\P_X(S_r) 
\ge 
    \P_{X^x_m} (S_r)
\ge
    \brb{ \P_X(S_r) }^m \;. \qedhere
\]
\end{proof}
We can now prove the last theorem of this section.
\begin{theorem}
\label{t:ass-nn}
If condition~\eqref{e:ass-nn} is satisfied and $x$ is a \leb{} point, then 
\[
    \E \bsb{ \babs{ \eta(X^x_m) - \eta(x) } } \to 0 \;, \qquad \text{as } m \to \iop \;.
\] 
\end{theorem}
\begin{proof}
Up to a rescaling, we can (and do) assume that $\babs{ \eta - \eta(x) } \le 1$.
Take any $\alpha$-sequence $(r_m)_{m\in \N, m\ge m_1}$, for some $\alpha\in (0,1)$.
By Theorem~\ref{t:victory-spheres}, it suffices to show that $\E \bsb{\I_{ S_{r_m} }(X^x_m) \, \babs{ \eta(X^x_m) - \eta(x) } } \to 0$, as $m\to \iop$. 
We will prove that this is the case by showing that each subsequence of $\brb{ \E \bsb{\I_{ S_{r_m} }(X^x_m) \, \babs{ \eta(X^x_m) - \eta(x) } } }_{m\in \N, m\ge m_1}$ has another subsequence that converges to zero.
Take any subsequence $(m_l)_{l \in \N}$ of $(m)_{m \in \N, m \ge m_1}$.
If there exists $L \in \N$ such that, for all $l\in \N$, if $l\ge L$, the following identity holds:
\[
    \E \Bsb{\I_{ S_{r_{m_l}} }(X^x_{m_l}) \babs{ \eta(X^x_{m_l}) - \eta(x) } }
=
    0 \;,
\]
\aaa{then} the claim is trivially true. 
Up to taking another subsequence, then, we can assume that, for all $l \in \N $,
\[
    \E \Bsb{\I_{ S_{r_{m_l}} }(X^x_{m_l}) \babs{ \eta(X^x_{m_l}) - \eta(x) } }
>
    0 \;,
\]
which in turn gives
\[
    \P \brb{ X^x_{m_l} \in S_{r_{m_l}} }
\ge
    \E \Bsb{\I_{ S_{r_{m_l}} }(X^x_{m_l}) \babs{ \eta(X^x_{m_l}) - \eta(x) } }
>
    0 \;.
\]
Moreover, recalling that $0 < r_m \to 0$ as $m\to \iop$ (by Lemma~\ref{l:incistato}), we have that there exists $M_1 \in \N$ such that, for all $l\in \N$, if $l\ge M_1$, then $r_{m_l} < R$.
Then, by Lemma~\ref{l:abs-cont-X-Xm} and condition \eqref{e:ass-nn}, we have, for each $l\in\N$, if $l\ge \max\{M, M_1\}$, 
\begin{align*}
    \E \Bsb{\I_{S_{r_{m_l}} } (X^x_{m_l}) \, \babs{ \eta(X^x_{m_l}) - \eta(x) } }
& \overset{\phantom{\eqref{e:ass-nn}}}{=}
    \E \Bsb{ \babs{ \eta(X^x_{\aaa{m_l}}) - \eta(x) } \mid X^x_{\aaa{m_l}} \in S_{r_{m_l}}}
    \,
    \P_{X^x_{m_l}} (S_{r_{m_l}})
\\
&\overset{\eqref{e:ass-nn}}{\le}
    C\,\E \Bsb{ \babs{ \eta(X) - \eta(x) } \mid X \in S_{r_{m_l}}}
    \,
    \P_{X^x_{m_l}} (S_{r_{m_l}})
\\
& \overset{\phantom{\eqref{e:ass-nn}}}{\le}
    C\,\E \Bsb{ \babs{ \eta(X) - \eta(x) } \mid X \in S_{r_{m_l}}}
\end{align*}
and thus, if $\E \Bsb{ \babs{ \eta(X) - \eta(x) } \mid X \in S_{r_{m_l}}} \to 0$ as $\aaa{l}\to \iop$, the theorem is proven.
Otherwise, up to taking another subsequence, we can (and do) assume that there exists $\e>0$ such that, for all $l\in \N$,
\[
    \E \Bsb{ \babs{ \eta(X) - \eta(x) } \mid X \in S_{r_{m_l}}} 
    \ge
    \e \;.
\]
Thus, for each $l\in \N$, we get
\begin{align*}
    0
& <
    \e \, \P_X \brb{ S_{r_{m_l}} } 
\le 
    \E \Bsb{ \babs{ \eta(X) - \eta(x) } \mid X \in S_{r_{m_l}}} 
    \,
    \P_X \brb{ S_{r_{m_l}} } 
\\
&
= 
    \E \Bsb{ \I_{S_{r_{m_l}}}(X) \, \babs{ \eta(X) - \eta(x) }} \;.
\end{align*}
Equivalently, for any $l\in \N$,
\[
    \frac{1}{\P_X \brb{ S_{r_{m_l}}}} 
\ge 
    \e \frac{1}{\E \Bsb{ \, \I_{S_{r_{m_l}}}(X) \babs{ \eta(X) - \eta(x) }}}
\]
which in turn implies
\begin{multline*}
    \frac{ \P_X \brb{B_{r_{m_l}} } }{ \P_X \brb{S_{r_{m_l}} } } + 1 
= 
    \frac{ \P_X \brb{B_{r_{m_l}} } + \P_X \brb{S_{r_{m_l}} } }{ \P_X \brb{S_{r_{m_l}} } } 
= 
    \frac{  \P_X \brb{\bar{B}_{r_{m_l}} } }{ \P_X \brb{S_{r_{m_l}} } } 
\\
\ge 
    \e \frac{ \P_X \brb{\bar{B}_{r_{m_l}} } }{ \E \Bsb{ \, \I_{S_{r_{m_l}}}(X) \babs{ \eta(X) - \eta(x) }} } 
\ge 
    \e \frac{ \P_X \brb{\bar{B}_{r_{m_l}} } }{ \E \Bsb{ \, \I_{\bar{B}_{r_{m_l}}}(X) \babs{ \eta(X) - \eta(x) }} }
\end{multline*}
and the last term diverges to infinity as $l \to \iop$ because $x$ is a \leb{} point and $0<r_{m_l} \to 0$ as $l\to \iop$.
Therefore, if $l\to \iop$
\[
    \frac{ \P_X \brb{B_{r_{m_l}} } }{ \P_X \brb{S_{r_{m_l}} } } 
    \to \iop\;,
\qquad \text{or equivalently } \qquad
    \frac{ \P_X \brb{S_{r_{m_l}} } }{ \P_X \brb{B_{r_{m_l}} } } 
    \to 0\;.
\]
Hence, there exists $K \ge 0$ such that, for all $l\in \N$, $\P_X(S_{r_{m_l}}) \le K \, \P_X(B_{r_{m_l}})$.
Furthermore, being $(r_{m_l})_{l \in \N}$ a subsequence of an $\alpha$-sequence and $x$ a \leb{} point, we have that $m_l \P_X(\bar{B}_{r_{m_l}}) \to \infty$ as $l \to \infty$ by \eqref{e:m-1}. 
Then, the assumptions of Lemma~\ref{l:ass-proba-tech} are satisfied, implying that $\E \bsb{\I_{S_{r_{m_l}} } (X^x_{m_l}) \, \babs{ \eta(X^x_{m_l}) - \eta(x) } }  \to 0$ as $l\to \iop$ and concluding the proof.
\end{proof}

\aaa{We} now illustrate with a concrete example that Assumption~\eqref{e:ass-nn} (hence Theorem~\ref{t:ass-nn}) can still hold even if the nearest neighbor rule is arbitrarily (but finitely) biased.

\begin{example}
\label{ex:counter-rev}
Let $\cX := \bcb{ -1, -\frac{1}{2}, -\frac{1}{3}, \l } \cup \{0\} \cup \bcb{ \l, \frac{1}{3}, \frac{1}{2}, 1 }$, $d$ be the distance induced on $\cX$ by the Euclidean metric, $\cB$ be the Borel $\sigma$-algebra of $(\cX, d)$, and $\mu \colon \cB \to [0,1]$ be the unique probability measure such that, for all $n\in \N$,
\[
    \mu\lrb{ \lcb{ -\frac{1}{n} } }
=
    \frac{1}{R} \frac{n-1}{n} \frac{1}{2^{2^n}}
\;,\quad
    \mu\lrb{ \lcb{ \frac{1}{n} } }
=
    \frac{1}{R} \frac{1}{n} \frac{1}{2^{2^n}}
\;, \quad\text{where }
    R
:=
    \sum_{n\in \N} \frac{1}{2^{2^n}} \;.
\]
Fix $C>1$ and $(p_n)_{n\in\N} \s (0,1)$ such that, for all $n\in \N$, $p_n \le (C-1)/n$. 
We define $\Omega = \cX \times \cX^{\N} \times \{0,1\}^\N$, $\cF = \cB \otimes \bigotimes_{k\in\N} \cB \otimes \bigotimes_{h \in \N} 2^{\{0,1\}}$, and $\P$ as the unique probability measure on $\cF$ such that for all $B, B_1, B_2, \l \in \cB$ and $\tht_1,\tht_2,\ldots \in \{0,1\}$,
\[
    \P\brb{ B\times B_1 \times B_2 \times \l \times \{\tht_1\} \times \{\tht_2\} \times \ldots }
= 
    \mu(B) \, \lrb{ \prod_{i=1}^{+\iop} \mu(B_i) } \prod_{n=1}^{+\iop} p_n^{\tht_n}(1-p_n)^{1-\tht_n} \;.
\]
It is well-known that such a probability measure exists \cite[Chapter VII, \aaa{Section~38}, Theorem~B]{halmos2013measure}. 
We define $X\colon \Omega \to \cX$, $(x, x_1, x_2, \l, \tht_1, \tht_2, \l) \mapsto x$, and for all $k\in \N$, $X_k \colon \Omega \to \cX$, $(x, x_1, x_2, \l, \tht_1, \tht_2, \l) \mapsto x_k$ and $\Theta_k \colon \Omega \to \{0,1\}$, $(x, x_1, x_2, \l, \tht_1, \tht_2, \l) \mapsto \tht_k$.
By definition, $(\Omega, \cF, \P)$ is a probability space. 
By construction, $X, X_1, X_2, \l$ are $\P$-independent random variables with common distribution $\P_X = \mu$, and they are $\P$-independent of $\Theta_1, \Theta_2, \ldots$, which are Bernoulli random variables with parameters $p_1,p_2,\ldots$ respectively.
Let
\[
    A
:=
    \cX \cap (0,\iop)\;,
\qquad
    \eta := \I_{A}\;,
\qquad\text{and }
    x := 0
    \;.
\]
Finally, for all $m\in \N$, let $X^x_m$ the nearest neighbor such that, whenever $\bcb{ -\frac{1}{n}, \frac{1}{n} } = \argmin_{ X' \in \{X_1, \l, X_m\} } X' $ for some $n\in \N$, then $X^x_m = \frac{1}{n}$ if and only if $\Theta_n=1$. 

\noindent At a high-level:
\begin{itemize}[topsep = 2pt, parsep = 2pt, itemsep = 2pt]
    \item given that $\argmin_{ X' \in \{X_1, \l, X_m\} } X' = \bcb{ -\frac{1}{n}, \frac{1}{n} }$ for some $n\in\N$, by definition of $\mu$, the odds of hitting $\frac{1}{n}$ against hitting $-\frac{1}{n}$ are approximately $1:n$, a fact that a fair tie-breaking rule should reflect; however, if we choose (for some $c>0$ and sufficiently large $n$'s) $p_n \ge c/n$, the odds become at least $c:n$; i.e., we can artificially increase by at least $c$-times the odds of picking positive (over negative) values, with $c$ large (if $C$ is also large);
    \item \eqref{e:ass-nn} holds because $p_n \le (C-1)/n$ poses a bound on the bias of the tie-breaking rule;
    \item \eqref{e:ass-proba} does not hold because the measure $\P_X(S_{1/n})$ of spheres $S_{1/n}$ decreases super-exponentially as $n \to \iop$;
    \item $x$ is a \leb{} point for $\eta$ because the measure $\P_X\brb{ \bar{B}_{1/n} }$ is more and more biased towards negative values as $n$ approaches $\iop$;
\end{itemize}
We begin by proving rigorously that $x$ is a \leb{} point.
For each $r>0$, defining $n_r := \min\{n\in\N \mid 1/n \le r\}$, we have that
\begin{multline*}
    \frac
    {
        \E \Bsb{ \I_{\bar{B}_r}(X)\, \babs{ \eta(X) - \eta(x) } }
    }
    {
        \P_X \brb{ \bar{B}_r }
    }
=
    \frac
    {
        \P_X \brb{ A \cap \bar{B}_r }
    }
    {
        \P_X \brb{ \bar{B}_r }
    }
=
    \frac
    {
        \sum_{k\ge n_r} \P \brb{ X =  \frac{1}{k} }
    }
    {
        \sum_{k\ge n_r} \P \lrb{ X = \pm\frac{1}{k} }
    }
\\
=
    \frac
    {
        \sum_{k\ge n_r} \frac{1}{k} \P \lrb{ X = \pm\frac{1}{k} }
    }
    {
        \sum_{k\ge n_r} \P \lrb{ X = \pm \frac{1}{k} }
    }
\le
    \frac
    {
        \sum_{k\ge n_r} \frac{1}{n_r} \P \lrb{ X = \pm \frac{1}{k} }
    }
    {
        \sum_{k\ge n_r} \P \lrb{ X = \pm \frac{1}{k} }
    }
=
    \frac{1}{n_r}
\to
    0, \quad \text{as } r \to 0^+\;,
\end{multline*}
hence $x$ is a \leb{} point.

While it is immediate that \eqref{e:ass-proba} does not hold, to show rigorously that \eqref{e:ass-nn} does requires some care.
Without loss of generality, we can (and do) assume that $r=1/n$, for some $n \in \N$.
Let $m,n\in \N$ and $E := \argmin_{ X' \in \{X_1, \l, X_m\} } X' = \bcb{ -\frac{1}{n}, \frac{1}{n} }$. 
Since
\begin{align*}
    \P\lrb{ E = \lcb{\frac{1}{n}} }
&
    =
    \sum_{\varnothing \neq I \s \{1,\ldots,m\}}
    \lrb{ \frac{1}{n} \P \lrb{ \labs{X} = \frac{1}{n} } }^{\labs{I}}
    \lrb{ \P\lrb{ \labs{X} > \frac{1}{n} } }^{m-\labs{I}} \;,
\\
    \P\lrb{ E \s \lcb{-\frac{1}{n}, \frac{1}{n}} }
&
    =
    \sum_{\varnothing \neq I \s \{1,\ldots,m\}}
    \lrb{ \P \lrb{ \labs{X} = \frac{1}{n} } }^{\labs{I}}
    \lrb{ \P\lrb{ \labs{X} > \frac{1}{n} } }^{m-\labs{I}} \;,
\end{align*}
we have that
\begin{align*}
&
    \E \Bsb{ \babs{ \eta(X^x_m) - \eta(x) } \mid X^x_m \in S_{1/n} }
=
    \frac{\P\Brb{ \brb{ E = \lcb{ \frac{1}{n} } } \cup \brb{ E=\lcb{-\frac{1}{n},\frac{1}{n} } \cap \Theta_n = 1 } }}{\P\brb{ E \s \lcb{ -\frac{1}{n}, \frac{1}{n} } }}
\\
&
\hspace{17.16469pt}
\le
    \frac{\P\lrb{ E = \lcb{ \frac{1}{n} }  }}{\P\lrb{ E \s \lcb{ -\frac{1}{n}, \frac{1}{n} } }} + \frac{C-1}{n} 
\le 
    \frac{1}{n} + \frac{C-1}{n} = C \E \Bsb{ \babs{ \eta(X) - \eta(x) } \mid X \in S_{1/n} }
\end{align*}
and so, by Theorem~\ref{e:ass-nn}, we know that $\E \bsb{ \babs{ \eta(X^x_m) - \eta(x) }} \to 0$ as $m\to \iop$.
\end{example}

In Theorems~\ref{t:ass-proba}~and~\ref{t:ass-nn} we proved that if conditions \eqref{e:ass-proba} or \eqref{e:ass-nn} are satisfied and $x$ is a \leb{} point, then $\E \bsb{\babs{ \eta(X^x_m) - \eta(x) } }  \to 0$, as $m\to \iop$.
The reader might be wondering if this implication holds with no assumptions other than that $x$ is a \leb{} point and $X^x_m$ is a nearest neighbor (among $X_1, \l, X_m$), i.e., if the result holds (in general) when neither condition \eqref{e:ass-proba} nor condition \eqref{e:ass-nn} are satisfied.
A modification of the previous example shows that this is not the case.

\begin{example}
\label{ex:counter}
Consider the same setting as in Example~\ref{ex:counter-rev}, where we now define, for all $m\in \N$, $X^x_m$ as the nearest neighbor such that, whenever $\bcb{ -\frac{1}{n}, \frac{1}{n} } \s \argmin_{ X' \in \{X_1, \l, X_m\} } X'$ for some $n\in \N$, then $X^x_m = \frac{1}{n}$. 

\noindent At a high-level:
\begin{itemize}[topsep = 2pt, parsep = 2pt, itemsep = 2pt]
    \item in contrast to the nearest neighbor defined in Example~\ref{ex:counter-rev}, here $X^x_m$ is ``infinitely'' biased towards positive values, breaking ties always in their favor.
    \item as in Example~\ref{ex:counter-rev}, $x$ is a \leb{} point for $\eta$ because the measure $\P_X\brb{ \bar{B}_{1/n} }$ is more and more biased towards negative values as $n$ approaches $\iop$;
    \item $\eta(X^x_m)$ does not converge to $\eta(x)$ in $\cL^1$ because of the interplay between the tie-breaking rule being biased towards the direction of positive values and the measure $\P_X(S_{1/n})$ of spheres $S_{1/n}$ decreasing super-exponentially as $n \to \iop$.
\end{itemize}
The same computation as in Example~\ref{ex:counter-rev} shows that $x$ is a \leb{} \aaa{point.}
We show now that $\eta(X^x_m)$ does not converge to $\eta(x)$ in $\cL^1$.
First, for all $m\in \N$, we have
\begin{multline*}
    \E \Bsb{ \babs{ \eta(X^x_m) - \eta(x) } }
=
    \E \Bsb{ \I_{A}(X^x_m) }
=
    \P ( X^x_m \in A )
=
    \sum_{n\in \N} \P \lrb{ X^x_m = \frac{1}{n} }
\\
\begin{aligned}
&\ge
    \sum_{n \in \N} \P \lrb{ \bigcup_{k=1}^m \lcb{ X_k = \frac{1}{n} } \cap \bigcap_{\substack{i=1\\i\neq k}}^m \lrb{ \lcb{ \labs{X_i} > \frac{1}{n} } \cup \lcb{ X_i = -\frac{1}{n} } } }
\\
&=
    m 
    \sum_{n \in \N} 
    \P \lrb{ X = \frac{1}{n} }
    \lrb{ \P \lrb{ \labs{X} > \frac{1}{n} } + \P \lrb{ X = -\frac{1}{n} } }^{m-1}
\\
&=
    m 
    \sum_{n \in \N} 
    \P \lrb{ X = \frac{1}{n} }
    \lrb{ \sum_{k=1}^{n-1} \P \lrb{ X = \pm \frac{1}{k} } 
    + \P \lrb{ X = -\frac{1}{n} } }^{m-1}
\\
&=
    m 
    \sum_{n \in \N} 
    \frac{1}{R} \frac{1}{n} \frac{1}{2^{2^n}}
    \lrb{ \sum_{k=1}^{n} 
    \frac{1}{R} \frac{1}{2^{2^k}} 
    - \frac{1}{R} \frac{1}{n} \frac{1}{2^{2^n}} }^{m-1}
\\
&=
    m 
    \sum_{n \in \N} 
    \frac{1}{R} \frac{1}{n} \frac{1}{2^{2^n}}
    \lrb{  
    \lrb{ 1 - \sum_{k=n+1}^{\iop} \frac{1}{R} \frac{1}{2^{2^k}} }
    - \frac{1}{R} \frac{1}{n} \frac{1}{2^{2^n}} }^{m-1}
\\
&\ge
    m 
    \sum_{n \in \N} 
    \frac{1}{R} \frac{1}{n} \frac{1}{2^{2^n}}
    \lrb{  
    1
    - \frac{1}{R} \frac{1}{n} \frac{1}{2^{2^n}} }^{m-1} \;.
\end{aligned}
\end{multline*}
This inequality will be used to prove that for countably many $m\in \N$, the quantity $ \E \bsb{ \babs{ \eta(X^x_m) - \eta(x) } }$ is bounded away from zero.
To do so, for all $j \in \N$, we fix an $m_j \in \N$ such that
\[
    \frac{1}{2m_j}
\le
    \frac{1}{R} \frac{1}{j} \frac{1}{2^{2^j}}
\le
    \frac{2}{m_j} \;.
\]
Note that $m_j \to \iop$ as $j\to \iop$. Then, for all $j\in \N$, we have
\begin{align*}
    \E \Bsb{ \babs{ \eta(X^x_{m_j}) - \eta(x) } }
&\ge
    m_j 
    \sum_{n \in \N} 
    \frac{1}{R} \frac{1}{n} \frac{1}{2^{2^n}}
    \lrb{  
    1
    - \frac{1}{R} \frac{1}{n} \frac{1}{2^{2^n}} }^{m_j-1}
\\
& >
    m_j 
    \frac{1}{R} \frac{1}{j} \frac{1}{2^{2^j}}
    \lrb{  
    1
    - \frac{1}{R} \frac{1}{j} \frac{1}{2^{2^j}} }^{m_j-1}
\\
& \ge 
    m_j 
    \frac{1}{2 m_j}
    \lrb{  
    1
    - \frac{2}{m_j} }^{m_j-1}
\\
& 
=
    \frac{1}{2}
    \lrb{  
    1
    - \frac{2}{m_j} }^{m_j-1} \to \frac{1}{2 e^{2}}
    \qquad \text{as } j\to \iop \;.
\end{align*}
This argument shows that there exist countably many $m\in \N$ such that the inequality $\E \Bsb{ \babs{ \eta(X^x_{m}) - \eta(x) } } \ge 1/(4e^2)$ holds, which in turn implies that $\eta(X^x_m)$ does \emph{not} converge to $\eta(x)$ in $\cL^1$.
\end{example}

\subsection{A broad class of tie-breaking rules}
\label{s:tie-break}

In this section we introduce a broad class of tie-breaking rules for which condition~\eqref{e:ass-nn} is always satisfied, regardless of how pathological $\P_X$ is. 
To do so, we introduce Independent Selectors of Indices of Minimum Numbers (or ISIMINs for short).
\begin{definition}[ISIMIN]
\label{d:ISIMIN}
We say that a sequence of pairs $(\smn_m, \Theta_m)_{m\in \N}$ is \aaa{an} \emph{Independent Selector of Indices of Minimum Numbers} (ISIMIN)\footnote{ISIMIN is pronounced ``easy-min''.} if, for all $m \in \N$, there exists a measurable space $(\cZ_m,\cG_m)$ such that $\Theta_m \colon \Omega \to \cZ_m$ is a random variable $\P$-independent from $X_1,\dots,X_m$ and $\smn_m \colon [0,\iop)^m \times \cZ_m \to \{1,\l,m\}$ is a measurable function satisfying, for all $r_1, \l , r_m \ge 0$ and all $z \in \cZ_m$,
\[
    \smn_m(r_1, \l, r_m, z) \in \argmin_{k \in \{1,\l,m\}} r_k \;.
\]
\end{definition}
At a high-level, the function $\smn_m$ selects the smallest among $m$ numbers, relying on an independent source $\Theta_m$ to break ties.

One of the classical ways to break ties in \nna{}s \citep{devroye1981almost} is lexicographically.
More precisely, ties are broken by selecting the smallest index among all indices minimizing the distance from $x$, i.e., 
\begin{equation}
    \label{e:det-choice}
    \min \lrb{ \argmin_{k \in \{1,\l,m\}} d(x,X_k) } \;.
\end{equation}
We can easily represent such a selection through \aaa{an ISIMIN} by taking, for all $m \in \N$, $\cZ_m := \{0\}$, $\cG_m := \bcb{ \{0\},\varnothing }$, $\Theta_m \equiv 0$,  and for all $r_1,\l,r_m \ge 0$,
\[
    \smn_m(r_1,\l,r_m, 0) := \min \lrb{ \argmin_{k \in \{1,\l,m\}} r_k } \;.
\]
We prove that such a sequence $(\smn_m, \Theta_m)_{m\in\N}$ is \aaa{an ISIMIN}.
With a slight abuse of notation, for all $m\in \N$ and all $i\in \{1,\l,m\}$, we denote the coordinate map $(r_1, \l, r_m, z) \mapsto r_i$ simply by $r_i$.
Then, for each $m\in \N$ and all $k \in \{1,\l,m\}$, we have that $\{ \smn_m = k \}$ is measurable,  being the Cartesian product of a measurable set (intersection of open and closed sets) with $\cZ_m$:
\[
    \{ \smn_m = k \} 
=
    \lrb
    {
        \bigcap_{j=1}^{k-1} \lcb{ r_j > r_k }
    \cap 
        \bigcap_{j=k+1}^{m} \lcb{ r_j \ge r_k }
    }
    \times
    \cZ_m \;.
\]
Therefore $\psi_m$ is measurable. Moreover, for each $m\in \N$, the constant function $\Theta_m \equiv 0$ is trivially $\P$-independent of $X_1, \l, X_m$.
Thus, the sequence $(\smn_m, \Theta_m)_{m\in \N}$ is \aaa{an ISIMIN}.
With this choice, we can reproduce the deterministic tie-breaking rule \eqref{e:det-choice} by taking, for each $m\in \N$,
\[
    \min \lrb{ \argmin_{k \in \{1,\l,m\}} d(x,X_k) }
=
    \smn_m \brb{ d(x,X_1), \l, d(x,X_m), \Theta_m  } \;.
\]

Another classical way to break ties is picking one of the closest points uniformly at random \citep{cerou2006nearest}.
More precisely, one draws a number $\tht_k$ in $[0,1]$ uniformly at random, and independently of everything else, for each $X_k$.
If there exist multiple $X_k \in \{X_1,\l,X_m\}$ with  $d(x,X_k) = \min_{j\in \{1,\l,m\}}d(x,X_j)$, then ties are broken by picking the smallest $X_k$ with the smallest value of $\tht_k$.
In the zero-probability event in which multiple closest $X_k$ have the same (smallest) value of $\tht_k$, one of them is chosen arbitrarily, e.g., lexicographically.
This selection can also be represented by \aaa{an ISIMIN}.
Indeed, for all $m\in \N$, let $\cZ_m := [0,1]^m$, $\cG_m$ be the Borel $\sigma$-algebra of $\cZ_m$, and we assume there exist $m$ uniform random variables $\tht_1, \l, \tht_m \colon \Omega \to [0,1]$ independent of $X_1, \l, X_m$ and also of each others (if they do not exist, one can simply enlarge the original probability space to accommodate them with a construction analogous to the one we present in Section~\ref{s:charact-arbitrary-measures}). 
Then for all $m\in \N$ we take $\Theta_m := (\tht_1, \l, \tht_m)$ and for all $r_1, \l, r_m \ge 0$,  $t_1, \l, t_m \in [0,1]$,
\[
    \smn_m(r_1, \l, r_m, t_1, \l, t_m) :=
    \min \lrb{ \argmin_{j \in  \underset{k\in \{1,\l,m\}}{\argmin} r_k} t_j } \;.
\]
We prove that such a sequence $(\smn_m, \Theta_m)_{m\in \N}$ is \aaa{an ISIMIN}. 
With a slight abuse of notation, for all $m\in \N$ and all $i\in \{1,\l,m\}$, we denote the coordinate map $(r_1, \l, r_m, t_1, \l, t_m) \mapsto r_i$ simply by $r_i$, and similarly, the coordinate map $(r_1, \l, r_m, t_1, \l, t_m) \mapsto t_i$ simply by $t_i$.
Then, for each integer $m\in \N$ and any $k\in \{1,\l,m\}$ we have that $\{\smn_m = k\}$ is measurable, being the union of an intersection of open and closed sets (in the following formula, think of $A$ as the set of indices that tie with $k$):
\[
        \bigcup_{\substack{A \s \{1,\l, m\} \\ A \not \ni k }}
        \lrb
        {
            \bigcap_{\substack{j \in \{1,\l,m\} \\ j \notin (\{k\}\cup A) } } \{r_k < r_j\}
            \cap
            \bigcap_{j\in A}  
            \{r_k = r_j\}
            \cap
            \bigcap_{\substack{j\in A \\ j < k}} 
            \{t_k < t_j\}
            \cap
            \bigcap_{ \substack{j\in A \\ j > k}} 
            \{t_k \le t_j\}
        } \;,
\]
where we recall that the intersection over an empty set of indices is the universe $[0,\iop)^m \times \cZ_m$. 
Therefore, for all $m\in \N$, the function $\psi_m$ is measurable.
Moreover, for each $m \in \N$, the function $\Theta_m = (\tht_1, \l, \tht_m)$ is $\P$-independent of $X_1, \l, X_m$ by construction.
Thus, the sequence $(\smn_m, \Theta_m)_{m\in \N}$ is a ISIMIN.
With this choice, we can reproduce the random tie-breaking rule in \citep{cerou2006nearest} by taking, for each $m\in \N$, 
\[
    \smn_m \brb{ d(x,X_1), \l, d(x,X_m), \Theta_m  } \;.
\]

\noindent We now define nearest neighbors according to arbitrary ISIMINs.

\begin{definition}[Nearest neighbors according to \aaa{an ISIMIN}]
\label{d:nn-selector}
Let $(\smn_m, \Theta_m)_{m\in \N}$ be \aaa{an ISIMIN}. 
For each $m \in \N$, the \emph{nearest neighbor} of $x$ (among $X_1, \l, X_m$), according to $(\smn_m, \Theta_m)$ is defined by
\[
    X^x_m 
:= 
    X_{\smn_m ( d(x,X_1), \l, d(x,X_m), \Theta_m )} \;.
\]
\end{definition}
Note that, being $(\smn_m, \Theta_m)_{m\in \N}$ a sequence of measurable pairs, nearest neighbors defined according to ISIMINs are also measurable, i.e., they are actually nearest neighbors according to Definition~\ref{n:tr-set-X'm}.
The next proposition shows that these nearest neighbors always satisfy an even stronger condition than \eqref{e:ass-nn}, regardless of pathological $\P_X$ is. 
The intuition behind it is that the distribution of $X^x_m$ is obtained by ``pulling'' the distribution of $X$ towards $x$.
Thus, if we prevent the ``pull'' towards $x$ by conditioning $X^x_m$ to have a constant distance from $x$, the distributions of $X^x_m$ and $X$ might coincide (at least if $X^x_m$ does not break ties with directional preferences, since, as we saw in Example~\ref{ex:counter}, directional preferences might distort the distribution of $X^x_m$ badly).
This is indeed the case for nearest neighbors defined according to ISIMINs.
Indeed, ISIMINs hide directions, basing decisions only on distances and (independent random choices of) indices.

\begin{proposition}
\label{p:ISIMIN}
Let $(\psi_m, \Theta_m)_{m\in\N}$ be \aaa{an ISIMIN}.
Assume that, for all $m\in \N$, the nearest neighbor $X^x_m$ (among $X_1,\l,X_m$) is defined according to $(\psi_m, \Theta_m)$.
If $r > 0$ is such that $\P_X ( S_r ) >0$,
then, for any Borel subset $A$ of $(\cX,d)$, it holds that
\[
    \P \brb{ X^x_m \in A \mid  X^x_m \in S_r } 
=
    \P\brb{ X \in A \mid X \in S_r } \;.
\]
This, in particular, implies that the condition \eqref{e:ass-nn} holds.
\end{proposition}
\begin{proof}
Let $r > 0$ be such that $\P_X ( S_r ) > 0$. Fix any $m\in \N$.
Recall that, by Lemma~\ref{l:abs-cont-X-Xm}, we have that $\P_{X^x_m} (S_r) > 0$.
Then, both conditional probabilities are well-defined.

To simplify the notation, we define the auxiliary functions
\[
    R_1:=d(x,X_1), \l, R_m:=d(x,X_m)
\]
and, for all $i,j \in \{1,\l,m\}$, we let $R_{i:j}:= R_i, R_{i+1}, \l, R_{j}$ with the understanding that $R_{i:j}$ does not appear if $j<i$.
With this notation, we have
\[
    X^x_m = X_{\smn_m (R_{1:m}, \Theta_m) }\;.
\]
For each $k \in \{ 1, \l,m \}$, note that $R_1, \l, R_{k-1}, X_k, R_{k+1}, \l, R_m, \Theta_m$ are $\P$-inde\-pendent random variables.
Then for each Borel set $A$ of $(\cX,d)$ we have that
{
\allowdisplaybreaks
\begin{align*}
& \P \brb{ X^x_m \in A \cap S_r }  
= \P \brb{ X_{\smn_m (R_{1:m}, \Theta_m) } \in A \cap S_r } \\
&\hspace{59.07991pt} = \sum_{k=1}^{m} \P \Brb{ \bcb{ X_{\smn_m (R_{1:m}, \Theta_m) }  \in A \cap S_r } \cap \bcb{ \smn_m (R_{1:m}, \Theta_m) = k } } \\
&\hspace{59.07991pt} = \sum_{k=1}^{m} \P \Brb{ \bcb{ X_k \in A \cap S_r } \cap \bcb{ \smn_m \lrb{ R_{1:k-1}, r, R_{k+1:m}, \Theta_m } = k } } \\
&\hspace{59.07991pt} = \sum_{k=1}^{m} \P \lrb{ X_k \in A \cap S_r } \, \P \brb{ \smn_m \lrb{ R_{1:k-1}, r, R_{k+1:m}, \Theta_m } = k } \\
&\hspace{59.07991pt} = \P_X \lrb{ A \mid S_r } \sum_{k=1}^{m} \P \lrb{R_k = r} \, \P \brb{ \smn_m \lrb{ R_{1:k-1}, r, R_{k+1:m}, \Theta_m} = k } \\
&\hspace{59.07991pt} = \P_X \lrb{ A \mid S_r } \sum_{k=1}^{m} \P \Brb{ \bcb{ \smn_m \lrb{ R_{1:k-1}, r, R_{k+1:m}, \Theta_m } = k } \cap \lcb{R_k = r} } \\
&\hspace{59.07991pt} = \P_X \lrb{ A \mid S_r } \sum_{k=1}^{m} \P \Brb{ \bcb{ \smn_m (R_{1:m}, \Theta_m) = k } \cap \lcb{ X_{\smn_m (R_{1:m}, \Theta_m) } \in S_r} }  \\
&\hspace{59.07991pt} = \P_X \lrb{ A \mid S_r } \, \P \brb{ X_{\smn_m (R_{1:m}, \Theta_m) }  \in S_r } = \P_X \lrb{ A \mid S_r } \, \P \brb{ X^x_m  \in S_r } \;,
\end{align*}
}

which in turn gives
\begin{equation*}
    \P \lrb{ X^x_m \in A \mid X^x_m \in S_r } 
= 
    \frac{ \P \lrb{ X^x_m \in A \cap S_r } }{ \P \lrb{ X^x_m  \in S_r } } 
= 
    \P_X\lrb{ A \mid S_r } \;.
\end{equation*}
Being $r$ and $m$ arbitrary, integrating we get the second part of the result.
\end{proof}

\subsection{\leb{} \texorpdfstring{$\iff$}{if and only if}  \nn}

In this section we collect some of the results we presented above in a theorem and a corollary. 
The theorem gives a characterization of \leb{} points in term of the $\cL^1$-convergence of $\eta(X^x_m)$ to $\eta(x)$ if some mild conditions are satisfied.
The corollary gives a concrete setting in which this characterization holds.

\begin{theorem}
\label{t:sum-up}
If condition~\eqref{e:ass-proba} or condition~\eqref{e:ass-nn} hold, the following are equivalent:
\begin{enumerate}[topsep = 2pt, parsep = 2pt, itemsep = 2pt]
    \item \label{i:leb-teo} $x$ is a \leb{} point;
    \item \label{i:nn-teo} $\eta(X^x_m)$ converges to $\eta(x)$ in $\cL^1$, as $m\to \iop$.
\end{enumerate}
\end{theorem}
\begin{proof}
The implication~\ref{i:nn-teo}~$\Rightarrow$~\ref{i:leb-teo} follows directly from Theorem~\ref{t:nn-then-leb}. 
The vice versa is a restatement of Theorems~\ref{t:ass-proba}~and~\ref{t:ass-nn}.
\end{proof}

Recall that ---while the implication \ref{i:nn-teo}~$\Rightarrow$~\ref{i:leb-teo} is always true--- if none of the conditions~\eqref{e:ass-proba}~and~\eqref{e:ass-nn} are satisfied, the implication~\ref{i:leb-teo}~$\Rightarrow$~\ref{i:nn-teo} does not hold in general (Example~\ref{ex:counter}).

We now state the aforementioned concrete version of the previous result.
\begin{corollary}
\label{c:powa}
If $(X^x_m)_{m\in\N}$ is defined according to \aaa{an ISIMIN} (Definitions~\ref{d:nn-selector}~and~\ref{d:ISIMIN}), the following are equivalent:
\begin{enumerate}[topsep = 2pt, parsep = 2pt, itemsep = 2pt]
    \item $x$ is a \leb{} point;
    \item $\eta(X^x_m)$ converges to $\eta(x)$ in $\cL^1$, as $m\to \iop$.
\end{enumerate}
\end{corollary}
\begin{proof}
The result follows immediately by Theorem~\ref{t:sum-up} and Proposition~\ref{p:ISIMIN}.
\end{proof}
As we previously noted in Section~\ref{s:tie-break}, the two common tie-breaking rules in \citep{devroye1981inequality} (ties are broken lexicographically) and \citep{cerou2006nearest} (ties are broken uniformly at random) are instances of ISIMINs.
Therefore, Corollary~\ref{c:powa} applies in particular to both cases.

Moreover, we remark that the lexicographic version of the $1$-\nna{}  can be formulated as a (memory/computationally-efficient) online algorithm (Algorithm~\ref{algo:onn}).
\begin{algorithm2e}
    \LinesNumbered
    \SetAlgoNoLine
    \SetAlgoNoEnd
    \DontPrintSemicolon
	\SetKwInput{kwInit}{Initialization}
	\KwIn{$x\in\cX$}
	\kwInit{let $R\gets\iop$ and $Y\gets 0$}
\For
    {
        $m=1,2,\l$ 
    }
    {
        observe $X_m$\;
        {
	    \If
	    {
	        $d(x,X_m) < R$
	    }
	    {
	        observe $\eta(X_m)$\;
	        let $R \gets d(x,X_m)$ and $Y \gets \eta(X_m)$\;
	    }
        }
    }
\caption{\label{algo:onn}Online \nn{}}
\end{algorithm2e}

In particular, this gives an \emph{online} characterization of \leb{} points: $x$ is a \leb{} point if and only if $Y \to \eta(x)$ in $\cL^1$.

\section{Beyond probability measures}
\label{s:gen}

In this section we present some consequences and extensions of the results we proved in Section~\ref{s:leb-vs-nn}, shifting the focus on the characterization of \leb{} points defined in terms of general measures.
We begin by fixing some notation and definitions that will be used throughout Sections~\ref{s:charact-arbitrary-measures}~and~\ref{s:seq-convergence}.

Let $(\cX, d)$ be a metric space and $x\in \cX$ an arbitrary point.
As per previous sections, we will denote for all $r>0$, the balls $\bar{B}_r(x)$ and $B_r(x)$ by $\bar{B}_r$ and $B_r$.
Let $\cB$ be the Borel $\sigma$-algebra of $(\cX,d)$,
$\eta \colon \cX \to \R$ a measurable function,
and $\mu\colon \cB \to [0,\iop]$ a (Borel) measure.

To avoid constant repetitions in our results, we now explicitly state our (mild) assumptions.

\begin{assumption}
Until the end of Section~\ref{s:gen}, we will assume the following:
\begin{enumerate} [topsep = 2pt, parsep = 2pt, itemsep = 2pt]
\item $x$ is in the \emph{support} of $\mu$, i.e., for all $r>0$, we have $\mu\brb{\bar{B}_r}>0$;
\item $\mu$ is \emph{locally-finite} at $x$, i.e., there exists an $R>0$ such that $\mu\brb{\bar{B}_R} < \iop$.
\end{enumerate}
\end{assumption}
From this point to the end of Section~\ref{s:gen}, we fix such an $R$ and we remark that all our results hold for any other $R' \in (0, R)$.

We now state the general definition of \leb{} points.

\begin{definition}[\leb{} point]
We say that $x$ is a \leb{} point (for $\eta$ with respect to $\mu$) if
\[
    \frac{1}{\mu\brb{\bar{B}_r}} \int_{\bar{B}_r} \babs{ \eta(x') - \eta(x) } \dif \mu(x')
    \to 0 \;,
    \qquad \text{as } r \to 0^+\;.
\]
As before, we call the ratios in the the previous definition \emph{\leb{} ratios}.
\end{definition}

\subsection{\leb{} points and nearest neighbors}
\label{s:charact-arbitrary-measures}

In this section we show how to build an instance of a \nna{} in order to obtain a characterization of \leb{} points in our general metric measure space $(\cX,d,\mu)$.

Take an arbitrary sequence of probability spaces $(\cZ_m, \cG_m, \nu_m)_{m\in\N}$.
Let 
\[
    \Omega 
:= 
    \cX \times \cX^\N \times \prod_{m\in\N} \cZ_m
\qquad\text{and}\qquad
    \cF
:=
    \cB \otimes \bigotimes_{k\in \N} \cB \otimes \bigotimes_{m\in \N} \cZ_m \;.
\]
Since $\mu\brb{\bar{B}_R(x)} < \iop$, we can define the probability measure 
\begin{align*}
    \nu\colon \cB & \to [0,1] \\
    A & \mapsto \frac{\mu\brb{ A \cap \bar{B}_R(x) } }{\mu\brb{\bar{B}_R(x)}}
\end{align*}
Let $\P \colon \cF \to [0,1]$ be the unique probability measure such that for all $A, A_1, A_2, \l \in \cB$, $G_1\in \cG_1, G_2 \in\cG_2, \l$, we have 
\[
    \P(A\times A_1 \times A_2 \times \l \times G_1 \times G_2 \times \ldots ) 
= 
    \nu(A) \, \nu(A_1) \, \nu(A_2) \l \nu_1(G_1) \, \nu_2(G_2) \l\;.
\]
It is well-known that such a probability measure exists \cite[Chapter VII, \aaa{Section~38}, Theorem~B]{halmos2013measure}. 
We define $X\colon \Omega \to \cX$, $(x, x_1, x_2, \l, z_1, z_2, \l) \mapsto x$, for all $k\in \N$, $X_k \colon \Omega \to \cX$, $(x, x_1, x_2, \l, z_1, z_2, \l) \mapsto x_k$, and for all $m\in \N$, $\Theta_m \colon \Omega \to \cZ_m$, $(x, x_1, x_2, \l, z_1, z_2, \l) \mapsto z_m$.
Finally, take an arbitrary sequence $(\psi_m)_{m\in \N}$ such that, for all $m\in \N$, $\psi_m \colon [0,\iop)^m \times \cZ_m \to \{1,\l,m\}$ is a measurable function satisfying, for all $r_1, \l , r_m \ge 0$ and all $z \in \cZ_m$,
$
    \smn_m(r_1, \l, r_m, z) \in \argmin_{k \in \{1,\l,m\}} r_k
$.

By construction, $X, X_1, X_2, \l$ are $\P$-independent random variables with common distribution $\P_X = \nu$ and $(\psi_m,\Theta_m)_{m\in\N}$ is \aaa{an ISIMIN} (Definition~\ref{d:ISIMIN}).

For the remainder of this section, for all $m\in \N$, $X^x_m$ will be the nearest neighbor (among $X_1, \l, X_m$) according to $(\psi_m,\Theta_m)$ (\aaa{Definition~\ref{d:nn-selector}}).

\begin{theorem}
\label{t:general-measures}
If the restriction of $\eta$ to $\bar{B}_R$ is bounded, the following are equivalent:
\begin{enumerate} [topsep = 2pt, parsep = 2pt, itemsep = 2pt]
    \item \label{i:lebt-a} $x$ is a \leb{} point for $\eta$ with respect to $\mu$;
    \item \label{i:nnt-b} $\E \bsb{\babs{ \eta(X^x_m) - \eta(x) } } \to 0$ as $m \to \infty$.
\end{enumerate}
\end{theorem}
\begin{proof}
By Corollary~\ref{c:powa}, condition \ref{i:nnt-b} is equivalent to $x$ being a \leb{} point for $\eta$ with respect to $\P_X$. Since, for all $r\in (0,R]$, we have
\begin{align*}
    \frac{1}{\mu\brb{\bar{B}_r}} \int_{\bar{B}_r} \babs{ \eta(x') - \eta(x) } \dif \mu(x')
&=
    \frac{1}{\mu\brb{\bar{B}_r}/\mu\brb{\bar{B}_R}} \int_{\bar{B}_r} \babs{ \eta(x') - \eta(x) } \dif \frac{\mu(x')}{\mu\brb{\bar{B}_R}}
\\
&=
    \frac{1}{\nu\brb{\bar{B}_r}} \int_{\bar{B}_r} \babs{ \eta(x') - \eta(x) } \dif \nu(x')
\\
&=
    \frac{\E \Bsb{ \I_{\bar{B}_r}(X) \, \babs{\eta(X) - \eta(x)} } }{\P_X \lrb{ \bar{B}_r }} \;,
\end{align*}
then $x$ a \leb{} point for $\eta$ with respect to $\mu$ if and only if $x$ is a \leb{} point for $\eta$ with respect to $\P_X$, which implies the result.
\end{proof}

\subsection{Sequential convergence}
\label{s:seq-convergence}

In this section we want to characterize \leb{} points in terms of convergence to zero of \leb{} ratios along suitable sequences of vanishing radii.
Note that $x$ is trivially a \leb{} point if $\mu\brb{ \{x\} } > 0$ (by the dominated monotonicity of measures) or if there exists an $r>0$ such that $\int_{\bar{B}_r }\babs{\eta - \eta (x)}\dif \mu = 0$.
Thus, for the remainder of this section we will focus on the non-trivial case in which $\mu\brb{ \{x\} }=0$ and for all $r>0$, $\int_{\bar{B}_r }\babs{\eta - \eta (x)}\dif \mu >0 $.

We now generalize to arbitrary measures the definition of $\alpha$-sequences introduced in Section~\ref{s:main-section-part-2} (Definition~\ref{d:alpha-seq-1}) which will play a central role in our sequential characterization of \leb{} points.
\begin{definition}[$\alpha$-sequence]
\label{d:alpha-seq}
Fix any $\alpha \in (0,1)$. For all $r>0$, define
\[
    M(r)
:=
    \bbrb{ \int_{\bar{B}_{r}} \babs{ \eta - \eta(x) } \dif \mu  }^{\alpha} 
    \Brb{ \mu \brb{ \bar{B}_r} }^{1-\alpha} \;.
\]
Let $R_1 \in (0, R)$ such that $M(R_1) < M(R)$ and $m_1 := \bce{ 1/ M(R_1) }$. 
For all $m \in \N$ such that $m \ge m_1$, define
\[
    r_m
:=
    \sup \lcb{ r > 0 \mid M(r) < \frac{1}{m} } \;.
\]
We say that $(r_m)_{m\in \N, m \ge m_1}$ is the \emph{$\alpha$-sequence} (for $\eta$ with respect to $\mu$) at $x$.
\end{definition}
Note that $\alpha$-sequences are well-defined, strictly positive, and vanishing by our assumptions that $\mu\brb{ \{x\} }=0$ and for all $r>0$, $\int_{\bar{B}_r }\babs{\eta - \eta (x)}\dif \mu >0 $ (as they were in Lemma~\ref{l:incistato}).

We now take a closer look to the proofs of Theorems~\ref{t:victory-spheres},~\ref{t:ass-proba},~\ref{t:ass-nn}. 
Note that, there, it is redundant to assume that $x$ is a \leb{} point. 
Indeed, we merely used the fact that \leb{} ratios with respect to closed (see after Definition~\ref{d:leb-point}) and open (the ratios appearing in \eqref{e:geom-measure-flavor}) balls are vanishing along the radii given by an $\alpha$-sequence, for some $\alpha \in (0,1)$.
By Theorem~\ref{t:nn-then-leb}, this proves that, at least when $\mu$ is a probability measure, a construction like the one we presented in Section~\ref{s:charact-arbitrary-measures} gives that $x$ is a \leb{} point if and only if \leb{} ratios with respect to both closed and open balls are vanishing along the radii given by an $\alpha$-sequence, for some $\alpha \in (0,1)$.
This is very surprising, since, in general, if \leb{} ratios (with respect to both closed and open balls) vanish along the radii given by a sequence, $x$ is not necessarily a \leb{} point, as shown by the following counterexample.

\begin{example}
Let $\cX := [0,1]$, $d$ be the Euclidean distance, $\mu$ be the \leb{} measure, and $x := 0$. 
For all $n\in \N$, define $\tht_n := \frac{1}{2^{2^n}}$. 
Consider the function $\eta := \sum_{n = 1}^\iop \I_{( \tht_{2n}, \tht_{2n-1}  ]}$. 
For any $m\in \N$, take $\rho_m := \tht_{2m}$ and $\rho'_m := \tht_{2m +1}$.
We show now that the \leb{} ratios with respect to both closed and open balls vanish along the sequence of radii $(\rho_m)_{m\in \N}$, but $x$ is not a \leb{} point since its \leb{} ratios do \emph{not} vanish along the sequence of radii $(\rho'_m)_{m\in \N}$.
Indeed, if $m\to \iop$, we have that
\begin{align*}
    \frac{1}{\mu\brb{ B_{\rho_m} }}
    \int_{B_{\rho_m}} \babs{ \eta - \eta(x) } \dif \mu
&=    
    \frac{1}{\mu\brb{ \bar{B}_{\rho_m} }}
    \int_{\bar{B}_{\rho_m}} \babs{ \eta - \eta(x) } \dif \mu
=
    \frac{1}{\tht_{2m}} \int_{0}^{\tht_{2m}} \eta(t) \dif t
\\
&=
    \frac{1}{\tht_{2m}} \int_{0}^{\tht_{2m+1}} \eta(t) \dif t
 \le
    \frac{\tht_{2m+1}}{\tht_{2m}}
=
    \frac{1}{2^{2^{2m}}}
\to
    0
\end{align*}
but, at the same time,
\begin{align*}
    \frac{1}{\mu\brb{ \bar{B}_{\rho'_m} }}
    \int_{\bar{B}_{\rho'_m}} \babs{ \eta - \eta(x) } \dif \mu
& =
    \frac{1}{\tht_{2m+1}} \int_{0}^{\tht_{2m+1}} \eta(t) \dif t
\ge
    \frac{1}{\tht_{2m+1}} \int_{\tht_{2m+2}}^{\tht_{2m+1}} 1 \dif t
\\
& =
    1- \frac{\tht_{2m+2}}{\tht_{2m+1}}
=
    1- \frac{1}{2^{2^{\aaa{2m+1}}}}
\to
    1\;.
\end{align*}
\end{example}
This highlights that $\alpha$-sequences are very special, because each one of them contains \aaa{in itself} enough information to characterize the convergence to zero of \leb{} ratios in the continuum.
We will prove that even more is true but before presenting the result, we give a handy definition.

\begin{definition}[\leb{} point along a sequence]
Take any $k\in \N$ and any vanishing sequence of strictly positive numbers $(\rho_m)_{m\in \N, m\ge k}$.
We say that $x$ is a \leb{} point (for $\eta$ with respect to $\mu$) along $(\rho_m)_{m\in \N, m\ge k}$ if
\[
    \frac{1}{\mu\brb{ \bar{B}_{\rho_m} }}
    \int_{\bar{B}_{\rho_m}} \babs{ \eta - \eta(x) } \dif \mu
\to 
    0\;,
\qquad \text{as } m \to \iop \;.
\]
\end{definition}

We remark that, unlike the results we proved in Section~\ref{s:leb-vs-nn}, the following theorem holds for more general measures and (possibly) unbounded integrands, with the minimal assumption that $\eta$ is locally-integrable around $x$.
The equivalence will be proved directly, without relying on nearest neighbor techniques.

\begin{theorem}
\label{t:charact-gen-measures}
If $\int_{\bar{B}_R} \labs{\eta} \dif \mu <\iop$, then the following are equivalent:
\begin{enumerate} [topsep = 2pt, parsep = 2pt, itemsep = 2pt]
\item $x$ is a \leb{} point along an $\alpha$-sequence, for some $\alpha\in (0,1)$;
\item $x$ is a \leb{} point.
\end{enumerate}
\end{theorem}

\begin{proof}
We prove the non-trivial implication.
Let $\alpha \in (0,1)$ and assume that $x$ is a \leb{} point along the $\alpha$-sequence $(r_m)_{m\in\N, m\ge m_1}$.
To lighten the notation, we define the auxiliary function
\begin{align*}
    f \colon \cX & \to \R\\
    x' & \mapsto \babs{ \eta(x') - \eta(x) } \;.
\end{align*}
With a straightforward adaptation of \eqref{e:M(r)}, \eqref{e:m-1}, and \eqref{e:m-2}, we can prove that for all $m\in \N$, $m\ge m_1$, we have
\begin{equation}
    \label{e:pizza}
    \bbrb{ \int_{B_{r_m}} f \dif \mu  }^{\alpha} \Brb{ \mu \brb{ {B}_{r_m}} }^{1-\alpha}
\le
    \frac{1}{m}
\le
    \bbrb{ \int_{\bar{B}_{r_m}} f \dif \mu  }^{\alpha} \Brb{ \mu \brb{ \bar{B}_{r_m} } }^{1-\alpha} \;.
\end{equation}
Assume by contradiction that there exists a vanishing sequence $(s_m)_{m\in \N}$ of strictly positive numbers and a $\delta>0$ such that, for all $m\in \N$, we have that
\begin{equation}
    \label{e:contra-proof}
    \frac{1}{\mu\brb{\bar{B}_{s_m}}} \int_{\bar{B}_{s_m}} f \dif \mu
\ge 
    \delta\;.
\end{equation}
Without loss of generality, we can assume that for all $m\in \N$,
\[
    \bbrb{ \int_{\bar{B}_{s_m}} f \dif \mu  }^{\alpha} \Brb{ \mu \brb{ \bar{B}_{s_m} } }^{1-\alpha}
<
    1 \;.
\]
For all $m\in \N$, let $n_m\in \N$ such that
\begin{equation}
    \label{e:n_m}
    \frac{1}{n_m+1}
\le
    \bbrb{ \int_{\bar{B}_{s_m}} f \dif \mu  }^{\alpha} \Brb{ \mu \brb{ \bar{B}_{s_m} } }^{1-\alpha}
<
    \frac{1}{n_m}\;.
\end{equation}
Note that $n_m \to \iop$ as $m\to \iop$, since the middle term vanishes as $m\to \iop$.
Hence, there exists $m_2\in \N$ such that for all $m\in \N$ with $m\ge m_2$, we have that $n_m \ge m_1$.
From the first inequality in \eqref{e:pizza} and the first inequality in \eqref{e:n_m}, for all $m \in \N$ such that $m\ge m_2$, we have that $r_{n_m+1} \le s_m$, which in turn gives
\begin{equation}
    \label{e:eq-a-caso}
    \mu \brb{ \bar{B}_{r_{n_m+1}}} \le \mu \brb{ \bar{B}_{s_m}}\;.
\end{equation}
Hence, for all $m \in \N$ such that $m\ge m_2$, we have that
\begin{align}
    \mu \brb{ \bar{B}_{r_{n_m+1}} } 
& \overset{\eqref{e:eq-a-caso}}{\le}
    \mu \brb{ \bar{B}_{\aaa{s_m}} } 
= 
    1 \cdot \mu \brb{ \bar{B}_{\aaa{s_m}} }\nonumber
\overset{\eqref{e:contra-proof}}{\le}
    \bbrb{ \frac{1}{\delta} \frac{1}{\mu\brb{\bar{B}_{s_m}}} \int_{\bar{B}_{s_m}} f \dif \mu }^\alpha 
    \mu \brb{ \bar{B}_{\aaa{s_m}} } \nonumber
\\ & =
    \frac{1}{\delta^\alpha}
    \bbrb{ \int_{\bar{B}_{s_m}} f \dif \mu }^\alpha 
    \Brb{ \mu \brb{ \bar{B}_{\aaa{s_m}} } }^{1-\alpha} 
\ \overset{\eqref{e:n_m}}{<} \
    \frac{1}{\delta^\alpha}
    \,
    \frac{1}{n_m} \;. \label{e:bagel}
\end{align}
Finally, for all $m \in \N$ such that $m\ge m_2$, we have that
\begin{multline*}
    \bbrb{ \frac{1}{\mu\brb{\bar{B}_{r_{n_m+1}}}} \int_{\bar{B}_{r_{n_m+1}}} f \dif \mu }^\alpha 
=
    \bbrb{ \int_{\bar{B}_{r_{n_m+1}}} f \dif \mu }^\alpha 
        \Brb{ \mu \lrb{ \bar{B}_{r_{n_m + 1}} } }^{1-\alpha}
    \frac{
        1
    }{
        \mu \brb{ \bar{B}_{r_{n_m + 1}} }
    }
\\ 
\overset{\eqref{e:pizza}}{\ge}
    \frac{1}{n_m+1}\,
    \frac{
        1
    }{
        \mu \brb{ \bar{B}_{r_{n_m + 1}} }
    }
\overset{\eqref{e:bagel}}{\ge}
    \delta^\alpha
    \,
    \frac{n_m}{n_m+1}
    \overset{m\to \iop}{\longrightarrow}
    \delta^\alpha > 0\;,
\end{multline*}
which, since $n_m \to \iop$ as $m\to \iop$, implies that $x$ is not a \leb{} point along a subsequence of $(r_m)_{m\in\N, m\ge m_1}$, contradicting the fact that $x$ is a \leb{} point along $(r_m)_{m\in\N, m\ge m_1}$.
\end{proof}

\section{From \leb{} points to \leb{} values}
\label{s:from-points-to-values}

In this brief section, we provide a straightforward but useful\footnote{For an application, see e.g., Example~\ref{ex:preiss}.} generalization of the previous results, shifting the focus from \leb{} points to \leb{} values.

\begin{definition}[Lebesgue value]
Let $(\cX, d )$ be a metric space, $\mu$ a locally-finite Borel measure of $(\cX, d )$, $\eta : \cX \to \R$ a locally-integrable function with respect to $\mu$ and $x \in \cX$ a point in the support of $\mu$. We say that $l \in \R$ is the \emph{Lebesgue value} of $\eta$ at $x$ (with respect to $\mu$) if
\[
    \frac{1}{\mu\brb{\bar{B}_r(x)}}\int_{\bar{B}_r(x)} \labs{\eta(x')-l} \dif\mu(x') \to 0 \;, \qquad r \to 0^+\;.
\]
\end{definition}
We point out that we use the word ``the'' in the definition of \leb{} values since if $l_1,l_2 \in \R$ are \leb{} values for $\eta$ at $x$ with respect to $\mu$, then the triangle inequality gives immediately $l_1 = l_2$.

We now make two key observations about \leb{} values.
The first one is that if $\mu \brb{ \{x\} } > 0$, then $\eta$ admits $\eta(x)$ as its \leb{} value at $x$.
The second one is that if $\mu \brb{ \{x\} } = 0$ and $\eta$ admits $l \in \R$ as its \leb{} value at $x$, then we can define
\begin{align*}
    \wtil \colon \cX  & \to \R\\
    x' & \mapsto
    \begin{cases}
        l & \text{if } x' = x \\
        \eta(x') & \text{otherwise}
    \end{cases}
\end{align*}
and obtain that $\wtil$ has a \leb{} point at $x$. 
With these two observations in mind, we can restate appropriately \emph{every condition and every result we obtained so far}, using \leb{} values instead of \leb{} points.

We illustrate with an example how this translation works for e.g., Corollary~\ref{c:powa}.
Consider the same setting as in Section~\ref{s:leb-vs-nn}.

\begin{corollary}
\label{c:powa-values}
If $(X^x_m)_{m\in\N}$ is defined according to \aaa{an ISIMIN} (Definitions~\ref{d:nn-selector}~and~\ref{d:ISIMIN}) and $l \in \R$, the following are equivalent:
\begin{enumerate}[topsep = 2pt, parsep = 2pt, itemsep = 2pt]
    \item $l$ is the \leb{} value for $\eta$ at $x$ with respect to $\P_X$;
    \item $\eta(X^x_m)$ converges to $l$ in $\cL^1$, as $m\to \iop$.
\end{enumerate}
\end{corollary}
\begin{proof}
Assume that $\eta(X^x_m)$ converges to $l$ in $\cL^1$, as $m\to \iop$.
If $\P_X\brb{ \{x\} }>0$ then $l = \eta(x)$ and the results follows directly from Corollary \ref{c:powa}. If $\P_X \brb{ \{x\} }=0$ then, since also $\P_{X^x_m}\brb{ \{x\} }=0$, using $\wtil$ as above, we have that
\[
    \E \bsb{\babs{\wtil(X^x_m) - \wtil(x)}} 
= 
    \E \bsb{\babs{\eta(X^x_m) - l}} \to 0\;, 
    \qquad m \to \infty\;, 
\]
and Corollary~\ref{c:powa} applied to $\wtil$ yields
\[
    \frac{\E \Bsb{ \I_{\bar{B}_r(x)}(X) \, \babs{\eta(X) - l} } }{\P_X \lrb{ \bar{B}_r(x) }} 
= 
    \frac{\E \Bsb{ \I_{\bar{B}_r(x)}(X) \, \babs{\wtil(X) - \wtil(x)} } }{\P_X \lrb{ \bar{B}_r(x) }} \to 0\;, 
    \qquad r \to 0^+\;.
\]
Vice versa, assume that $l$ is the \leb{} value for $\eta$ at $x$ with respect to $\P_X$. If $\P_X \brb{ \{x\} }>0$ then $l = \eta(x)$ and the results follows from Corollary \ref{c:powa}. If $\P_X \brb{ \{x\} }=0$ then, using $\wtil$ as above,
\[
    \frac{\E \Bsb{ \I_{\bar{B}_r(x)}(X) \, \babs{\wtil(X) - \wtil(x)} } }{\P_X \lrb{ \bar{B}_r(x) }} 
= 
    \frac{\E \Bsb{ \I_{\bar{B}_r(x)}(X) \, \babs{\eta(X) - l} } }{\P_X \lrb{ \bar{B}_r(x) }} \to 0\;,
    \qquad r \to 0^+\;,
\]
and since also $\P_{X^x_m} \brb{ \{x\} }=0$, Corollary~\ref{c:powa} gives us
\[
    \E \bsb{\babs{\eta(X^x_m) - l}} 
= 
    \E \bsb{\babs{\wtil(X^x_m) - \wtil(x)}} \to 0\;,
    \qquad m \to \infty\;. \qedhere
\]
\end{proof}

\section{ISIMINs and nearest neighbor classification}
\label{s:classification}

In this section we present applications of ISIMINs to nearest neighbor classification. 
Our results extend what is known for lexicographical tie-breaking rules \citep{devroye1981inequality} to the more general tie-breaking rules given by ISIMINs. 
In particular, this implies that our results hold for the (uniformly) random tie-breaking rule in \citep{cerou2006nearest}.

We consider the same setting as in Section~\ref{s:leb-vs-nn} (more specifically, of Subsection~\ref{s:tie-break}), with the following differences:
\begin{enumerate}[topsep = 2pt, parsep = 2pt, itemsep = 2pt]
    \item there is no fixed $x \in \cX$;
    \item there exists a sequence of random variables $Y_1, Y_2, \l\colon \Omega \to \{0,1\}$ such that $(X,Y),(X_1,Y_1),(X_2,Y_2),\l$ are $\P$-i.i.d.;
    \item $(\psi_m, \Theta_m)_{m\in \N}$ is \aaa{an ISIMIN} such that for each $m \in \N$ we have that $\Theta_m$ is $\P$-independent of $(X,Y),(X_1,Y_1),\l,(X_m,Y_m)$;
    \item rather than being an arbitrary bounded function, $\eta \colon \Omega \to \R$ is the \emph{regression function} of $Y$ with respect to $X$, i.e., $\eta$ is any measurable function such that $\eta(X) = \E[ Y \mid X]$ (whose existence is guaranteed by the Doob--Dynkin Lemma \citep[Chapter~1.2, Proposition~3]{rao2006probability}), that we can assume $[0,1]$-valued.
\end{enumerate}
We now define nearest neighbor classification by means of \aaa{an ISIMIN}.
\begin{definition}[Nearest neighbor classification according to \aaa{an ISIMIN}]
For all $m\in \N$, the nearest neighbor classification (with training set $(X_1,Y_1),\l(X_m,Y_m)$ and test data point $X$) according to $(\psi_m, \Theta_m)$ is defined by $Y_{\Smn_m}$,
where $\Smn_m$ is the random index of the random point among $X_1, \l, X_m$ that is closest to $X$ according to $(\smn_m, \Theta_m)$, i.e., $\Smn_m := \smn_m \brb{ d(X,X_1), \l, d(X,X_m), \Theta_m}$.
\end{definition}
We will prove the convergence of the classification risk $\P(Y_{\Smn_m} \neq Y)$ of the nearest neighbor $Y_{\Smn_m}$ whenever the \leb{}--\besi{} differentiation theorem holds for $\eta$, i.e., if $\P_X$-almost every $x\in \cX$ is a \leb{} point for $\eta$.
This will follow from our results on ISIMINs and Proposition~\ref{p:prop-inf}.
This proposition is similar in spirit to other known results for plug-in decisions, but it is tailored to our $1$-\nn{} classification problem.
E.g., in \citep[Theorem~2.2]{devroye1996probabilistic}, the comparison term on the right hand side is the Bayes risk $\E\bsb{ \min\brb{ \eta(X),1-\eta(X) } }$, since the goal there is to obtain consistency. 
However, it has long been known that the $1$-\nna{} is not consistent without additional assumptions (see, e.g., \citep[Theorem~5.4 and subsequent remark]{devroye1996probabilistic}).
Nevertheless, one can study the convergence of its risk.
With this goal in mind, the appropriate quantity to compare $\P(Y_{\Smn_m} \neq Y)$ to is not the Bayes risk, but rather a ``surrogate risk'', which turns out to be $2 \E\bsb{ \eta(X) \brb{ 1-\eta(X) } }$.
A similar formula appears also in \citep{devroye1981inequality}.
\begin{proposition}
\label{p:prop-inf}
For all $m\in \N$, we have
\[
    \biggl\lvert \P \brb{ Y_{\Smn_m} \neq Y } - 2\E \Bsb{ \eta(X)\brb{ 1-\eta(X) } } \biggr\rvert 
\le 
    \int_{\cX} \E \Bsb{\babs{ \eta(X^x_m) - \eta(x) } } \dif\P_X(x) \;,
\]
where for each $x \in \cX$, $X^x_m$ is the nearest neighbor of $x$ (among $X_1, \l, X_m$), according to $(\smn_m, \Theta_m)$, i.e.,  
$
    X^x_m 
:= 
    X_{\smn_m ( d(x,X_1), \l, d(x,X_m), \Theta_m )}
$.
\end{proposition}

\begin{proof}
Fix any $m\in \N$.
Note that
\[
    \P \brb{ Y_{\Smn_m} \neq Y }
=
    \P \Brb{ \bcb{ Y_{\Smn_m} = 1 } \cap \bcb{ Y = 0 } } + \P \Brb{ \bcb{ Y_{\Smn_m} = 0 } \cap \bcb{ Y = 1 } } \;.
\]
We analyze the first term. The second one can be computed similarly.
Note that
\[
    \P \Brb{ \bcb{ Y_{\Smn_m} = 1 } \cap \bcb{ Y = 0 } } 
= 
    \E\bsb{ Y_{\Smn_m}(1-Y) } = \E\Bsb{\E\bsb{ Y_{\Smn_m}(1-Y) \mid X } }
=
    (\star) \;.
\]
Now, if we let $\cV = \cX$, $\cW = \lrb{\cX\times\{0,1\}}^m\times\cZ_m$, $U = 1-Y$, $V = X$, $W = \brb{ (X_1,Y_1), \l, (X_m,Y_m), \Theta_m }$, and 
\begin{align*}
    f \colon \cV \times \cW & \to [0,1] \:, \\ \Brb{x,\brb{(x_1,y_1),\l,(x_m,y_m),z_m } } & \mapsto y_{ \smn_m(d(x,x_1), \l, d(x,x_m),z_m) } \;,
\end{align*}
applying Lemma~\ref{l:fake-indep} (see Appendix~\ref{s:proba}), we get
\begin{align*}
    \E\bsb{ (1-Y)Y_{\Smn_m} \mid X }
&= 
    \E\bsb{ Uf(V,W) \mid V } 
\overset{\eqref{e:generalized-conditional-independence}}{=} 
    \E [ U \mid V ] \, \E\bsb{ f(V,W) \mid V } 
\\
&= 
    \E\bsb{ 1-Y \mid X } \, \E\bsb{ Y_{\Smn_m} \mid X } 
= 
    \brb{1 - \E\bsb{ Y \mid X } } \, \E\bsb{ Y_{\Smn_m} \mid X } 
\\
&= 
    \brb{1 - \eta(X) } \, \E\bsb{ Y_{\Smn_m} \mid X } 
= 
    \E\Bsb{ Y_{\Smn_m} \, \brb{1 - \eta(X) } \mid X  } \;.
\end{align*}
Thus
\begin{align*}
    (\star)
&=
    \E \bbsb{ \E\Bsb{ Y_{\Smn_m} \, \brb{1 - \eta(X) }  \mid X } }
=
    \E\Bsb{ Y_{\Smn_m} \, \brb{1 - \eta(X) } }
\\
&=
    \sum_{k=1}^m \E\Bsb{ Y_{k} \, \brb{1 - \eta(X) } \, \I_{\{\Smn_m = k \}} }
\\
&=
    \sum_{k=1}^m \E \bbsb{ \E\Bsb{Y_{k} \, \brb{1 - \eta(X) } \, \I _{\{\Smn_m = k \}} \mid X_k } }
=
    (\star \star)
\end{align*}
Now, if for each $k\in \{1,\l,m\}$, we let $\cV = \cX$, $\cW = {\cX}^m \times \cZ_m$, $U = Y_k$, $V = X_k$, $W = (X, X_1, \l, X_{k-1}, X_{k+1}, \l, X_m, \Theta_m)$, applying Lemma~\ref{l:fake-indep} (Appendix~\ref{s:proba}) to the function
\begin{align*}
    f \colon \cV \times \cW & \to [0,1] \:, 
\\ 
    \Brb{ x_k,\brb{(x,x_{1:k-1},x_{k+1:m}),z_m} }  & \mapsto \brb{ 1-\eta(x) } \, \I_{\{\smn_m = k \}} \brb{ d(x,x_1), \l, d(x,x_m),z_m } \;,
\end{align*}
where $x_{1:k-1}$ (resp., $x_{k+1:m}$) is a shorthand for $x_1, \l, x_{k-1}$ (resp., $x_{k+1}, \l, x_m$) ---with the obvious adjustments for $k=1$ (resp., $k=m$)--- yields
\begin{align*}
    \E\Bsb{Y_{k} \, \brb{1 - \eta(X) } \,  \I_{\{\Smn_m = k \}} \mid X_k }
&= 
    \E\bsb{ U f(V,W) \mid V } 
\overset{\eqref{e:generalized-conditional-independence}}{=} 
    \E [ U \mid V ] \E\bsb{ f(V,W) \mid V }
\\
&= \E [ Y_{k} \mid X_k ] \, \E\Bsb{ \brb{1 - \eta(X) } \,  \I_{ \{\Smn_m = k \} } \mid X_k }
\\
&= \eta(X_{k}) \, \E\Bsb{ \brb{1 - \eta(X) } \,  \I_{ \{\Smn_m = k \} } \mid X_k }
\\
&= \E\Bsb{ \eta(X_{k}) \, \brb{1 - \eta(X) } \, \I_{ \{\Smn_m = k \} } \mid X_k } \;.
\end{align*}
Thus
\begin{align*}
    (\star\star)
&=
    \sum_{k=1}^m \E \bbsb{ \E\Bsb{  \eta(X_{k}) \, \brb{1 - \eta(X) } \, \I_{ \{\Smn_m = k \} } \mid X_k } }
\\
&=
    \sum_{k=1}^m \E\Bsb{ \eta(X_{k}) \, \brb{1 - \eta(X) } \, \I_{ \{\Smn_m = k \} } }
\\
&=
    \sum_{k=1}^m \E\Bsb{ \eta(X_{\Smn_m}) \, \brb{1 - \eta(X) } \, \I_ { \{\Smn_m = k \} } }
=
    \E\Bsb{ \eta(X_{\Smn_m}) \, \brb{1 - \eta(X) } } \;.
\end{align*}
So
\[
    \P \Brb{ \bcb{ Y_{\Smn_m} = 1 } \cap \bcb{ Y = 0 } } 
= 
    \E\Bsb{ \eta(X_{\Smn_m})\brb{1 - \eta(X) } } \;.
\]
Analogously, we can prove
\[
    \P \Brb{ \bcb{ Y_{\Smn_m} = 0 } \cap \bcb{ Y = 1 } } 
= 
    \E\Bsb{ \eta(X)\brb{1 - \eta(X_{\Smn_m}) } } \;.
\]
Hence
\begin{multline*}
    \biggl\lvert \P \brb{ Y_{\Smn_m} \neq Y } - 2\E \Bsb{ \eta(X)\brb{ 1-\eta(X) } } \biggr\rvert
\\
    \begin{aligned}
&=
    \biggl\lvert \E\Bsb{ \eta(X_{\Smn_m}) \brb{1 - \eta(X) } } + \E\Bsb{ \eta(X)\brb{1 - \eta(X_{\Smn_m}) } } - 2\E \Bsb{ \eta(X)\brb{ 1-\eta(X) } } \biggr\rvert
\\
&=
    \biggl\lvert \E\Bsb{ \brb{\eta(X_{\Smn_m}) - \eta(X) }\brb{1 - \eta(X) } + \eta(X) \Brb{ \brb{1 - \eta(X_{\Smn_m}) } - \brb{1 - \eta(X) } } }  \biggr\rvert
\\
&=
    \biggl\lvert \E\Bsb{ \brb{1 - 2\eta(X) }\brb{\eta(X_{\Smn_m}) - \eta(X) } } \biggr\rvert 
    \le 
    \E\bbsb{ \Babs{ \brb{1 - 2\eta(X) }\brb{\eta(X_{\Smn_m}) - \eta(X) } } }
\\
&\le
    \E\Bsb{ \babs{ \eta(X_{\Smn_m}) - \eta(X) } }
=
    \E \bbsb{ \E\Bsb{ \babs{ \eta(X_{\Smn_m}) - \eta(X) } \mid X } }
\\
&=
    \E \bbsb{ \E\Bsb{ \babs{ \eta(X_{\smn_m(d(X,X_1),\l,d(X,X_m),\Theta_m)}) - \eta(X) } \mid X } }
\\
& \overset{\eqref{e:freezing}}{=}
   \E\Bbsb{ \bbsb{\E\Bsb{ \babs{ \eta(X_{\smn_m(d(x,X_1),\l,d(x,X_m),\Theta_m)}) - \eta(x) } } }_{x=X} }
\\
&=
    \E\Bbsb{ \bbsb{\E\Bsb{ \babs{ \eta(X^x_m) - \eta(x) } } }_{x=X} }
=
    \int_{\cX} \E\Bsb{ \babs{ \eta(X^x_m) - \eta(x) } } \dif\P_X(x) \;,
\end{aligned}
\end{multline*}
where the second to last inequality follows by the Freezing Lemma (Lemma~\ref{l:freezing} in Appendix~\ref{s:proba}).
\end{proof}
We can now state the most important result of this section.
It guarantees the convergence of the classification risk of the $1$-\nna{}, using ISIMINs, on arbitrary metric spaces, \aaa{with the only assumption that $\P_X$-almost every $x \in \cX$ is a \leb{} point for the regression function $\eta$ with respect to $\P_X$.}
\begin{theorem}
\label{t:conv-nn-classification}
If $\P_X$-almost every $x \in \cX$ is a \leb{} point for $\eta$ with respect to $\P_X$, then
\begin{equation}
    \P \brb{ Y_{\Smn_m} \neq Y } \to 2\E \Bsb{ \eta(X)\brb{ 1-\eta(X) } } \;, \qquad{} m \to \infty \;.
\end{equation}
\end{theorem}
\begin{proof}
As in Proposition~\ref{p:prop-inf}, we define, for each $m\in \N$ and any $x \in \cX$, $X^x_m$ as the nearest neighbor of $x$ (among $X_1, \l, X_m$), according to $(\smn_m, \Theta_m)$, i.e.,  
$
    X^x_m 
:= 
    X_{\smn_m ( d(x,X_1), \l, d(x,X_m), \Theta_m )}
$.
Take $N \subset \cX$ be a Borel set of $(\cX,d)$ such that $\P_X(N) = 0$ and, for all $x \in \cX \m N$, we have that $x$ is a \leb{} point.
Then, Proposition~\ref{p:prop-inf} implies, for all $m \in \N$
\begin{multline*}
    \biggl\lvert \P \brb{ Y_{\Smn_m} \neq Y } - 2\E \Bsb{ \eta(X)\brb{ 1-\eta(X) } } \biggr\rvert 
\le 
    \int_{\cX} \E \Bsb{\babs{ \eta(X^x_m) - \eta(x) } } \dif\P_X(x)
\\
=
    \int_{\cX\m N} \E \Bsb{\babs{ \eta(X^x_m) - \eta(x) } } \dif\P_X(x)
    \to 0 \;, \qquad \text{as } m \to \iop
    \;,
\end{multline*}
where the last term vanishes by Corollary~\ref{c:powa} and the dominated convergence theorem.
\end{proof}
The previous result gives immediately the consistency of the $1$-\nn{} classifier $Y_{\Smn_m}$ in arbitrary metric spaces under the \emph{realizability assumption}
\begin{equation}
    \label{e:realizability-assumption}
    \exists f\colon \cX \to \{0,1\}, f \text{ measurable, such that } \P \brb{ f(X) = Y } = 1\;.
\end{equation}
Before stating the result, we recall that a classification algorithm is said $\P$-\emph{consistent} if its classification risk converges to the \emph{Bayes risk}
\[
    L^{\star}:=\E \Bsb{ \min \Brb{\eta(X),\brb{ 1-\eta(X) }}  }\;.
\]
In our setting, the $\P$-consistency condition for $Y_{\Smn_m}$ is
\[
    \P\brb{ Y_{\Smn_m} \neq Y } \to L^{\star} \;, \qquad m \to \iop \;.
\]
\begin{corollary}
\label{c:const-realizz}
Assume that $\P_X$-almost every $x \in \cX$ is a \leb{} point for $\eta$ with respect to $\P_X$.
Then, the following are equivalent:
\begin{enumerate}[topsep = 2pt, parsep = 2pt, itemsep = 2pt]
    \item $Y_{\Smn_m}$ is $\P$-consistent and $\P\brb{\eta(X) = 1/2} = 0$;
    \item the realizability assumption holds.
\end{enumerate}
\end{corollary}
\begin{proof}
We know from the previous theorem that the classification risk of $Y_{\Smn_m}$ converges to $2\E \bsb{ \eta(X)\brb{ 1-\eta(X) } }$, as $m \to \infty$. 
Since for any $y \in \brb{ 0, \frac{1}{2} } \cup \brb{ \frac{1}{2}, 1 }$ it holds that $2y (1-y) > \min (y,1-y)$, to have $2\E \bsb{ \eta(X)\brb{ 1-\eta(X) } } = L^{\star}$ it is necessary and sufficient that $\P \brb{\eta(X) \in \bcb{0,\frac{1}{2},1 }} = 1$. 
Note that, since $Y$ is $\{0,1\}$-valued, there exists $F \in \cF$ such that $\I_F = Y$.

Suppose that the realizability assumption holds. Let $f\colon \cX \to \{0,1\}$ be a measurable function such that $\P \brb{ f(X) = Y } = 1$. Since
\[
    \eta(X) = \E[Y \mid X] = \E[f(X) \mid X] = f(X) = Y = \I_F \;, \qquad \P\text{-almost surely} \;,
\]
it holds that $\P \brb{ \eta(X) \in \{0,1\} } = 1$. 
Thus, we have that $Y_{\Smn_m}$ is $\P$-consistent and $\P\brb{\eta(X) = 1/2} = 0$.

Vice versa, suppose that $Y_{\Smn_m}$ is $\P$-consistent and $\P\brb{\eta(X) = 1/2} = 0$. 
Then $\P \brb{ \eta(X) \in \lcb{ 0, 1 } } = 1$. Let $E := \brb{\eta(X)}^{-1}\brb{\{1\}}$. 
Since $E$ is $\sigma\brb{ \eta(X) }$-measurable, it is also $\sigma(X)$-measurable.
Then, by definition, there exists a Borel subset $A$ of $(\cX,d)$ such that $E = X^{-1}(A)$, and so $\I_E = \I_{X^{-1}(A)} = \I_A(X)$. We show that the realizability assumption holds with $f = \I_A$. Now, note that
\[
    \P(E) = \E[\I_E] = \E[\I_E\I_E] = \E\bsb{\I_E \E[Y \mid X]} = \E[\I_E Y] = \E[\I_E \I_F] = \P(E \cap F)\;,
\]
where we used the fact that $\P \brb{ \I_{E} = \eta(X) } = 1$. 
Furthermore, note that
\[
    0 = \E[0] = \E[\I_{E^c} \I_E] = \E\bsb{\I_{E^c} \E[Y \mid X]} = \E[\I_{E^c} Y] = \E[\I_{E^c} \I_F] = \P(E^c \cap F)\;.
\]
Then
\begin{align*}
    \E\Bsb{ \babs{\I_A(X) - Y} } &= \E\bsb{ \labs{\I_E - \I_F} } = \P(E \cap F^c) + \P(E^c \cap F)
\\
    &= \P(E \cap F^c) = \P(E) - \P(E\cap F) = 0 \;,
\end{align*}
that is equivalent to $\P \lrb{\I_A(X) = Y} = 1 $.
\end{proof}
As we pointed out in the introduction, if $(\cV,d_{\cV})$ is a Euclidean space then $\mu$-almost every $v \in \cV$ is a \leb{} point for $f$ with respect to $\mu$, for every Borel probability measure $\mu$ and each bounded measurable function $f \colon \cV \to \R$. 
This is a consequence of \leb{}--\besi{} differentiation theorem \cite[Theorems~1.32-1.33]{evans2015measure}. The same holds true if $(\cV,d_{\cV})$ is a finite dimensional Banach space \citep{loeb2006microscopic}, a locally-compact separable ultrametric space \cite[Theorem 9.1]{simmons2012conditional}, a separable Riemannian manifold \cite[Theorem 9.1]{simmons2012conditional}, or the straightforward case where $\cV$ is (at most) countable.
In particular, if $(\cX,d)$ is one of the previous metric spaces, then $\P_X$-almost every $x \in \cX$ is a \leb{} point for $\eta$ with respect to $\P_X$, regardless of which specific Borel probability measure $\P_X$ is.
This does not hold in every metric space. 
Indeed, \cite{preiss1979invalid} showed that if $(\cV,d_{\cV})$ is an infinite dimensional separable Hilbert space, then counterexamples exist, even if the underlying Borel probability measure is Gaussian. 
\cite{preiss1983dimension} also characterized the metric spaces where \leb{}--\besi{} differentiation theorem holds true, in terms of a notion of $\sigma$-finite dimensionality of the space.
\cite{cerou2006nearest} used Preiss' counterexample to show that the $k_m$-\nn{} classification algorithm is not necessarily consistent when the \leb{}--\besi{} differentiation theorem does not hold. 
The same idea works in our context.

\begin{example}
\label{ex:preiss}
\cite{preiss1979invalid} showed that there exists a Polish metric space $\lrb{ \cX, d }$ (which is actually an infinite dimensional separable Hilbert space), a Borel probability measure $\mu$ (which is actually Gaussian) and a compact set $K$ of $\lrb{\cX, d}$ such that $\mu(K) > 0$ and, for all $x\in \cX$, $x \in \supp(\mu)$ and
\[
    \frac{ 1 } { \mu \lrb{\bar{B}_r(x)} } \int_{\bar{B}_r(x)} \I_K \dif \mu \to 0 \;,
    \qquad \text{as } r \to 0^+\;.
\]
Let $(\Omega, \cF, \P)$ be a probability space, $X, X_1, X_2, \l$ be $\P$-i.i.d.\ random variables with common distribution $\P_X = \mu$, and $(\smn_m, \Theta_m)_{m\in \N}$ be \aaa{an ISIMIN} (for the existence of this setting, see Section~\ref{s:charact-arbitrary-measures}).
Define $Y := \I_K(X), Y_1 := \I_K (X_1), Y_2 := \I_K (X_2), \l$. 
As before, we denote for any $m\in \N$ and all $x\in \cX$, the nearest neighbor (among $X_1,\l,X_m$) according to $(\smn_m,\Theta_m)$ by $X^x_m := X_{\smn( d(x,X_1), \l, d(x,X_m), \Theta_m ) }$. 
Note that, for all $x \in \cX$, we have that
\[
    \frac{\E \Bsb{\I_{\bar{B}_r(x)}(X) \babs{\I_{K}(X)-0}}}{\P_X(\bar{B}_r(x))} = \frac{ 1 } { \mu \lrb{\bar{B}_r(x)} } \int_{\bar{B}_r(x)} \I_{K} \dif \mu \to 0 \;, 
    \qquad \text{as } r \to 0^+\;.
\]
By Corollary~\ref{c:powa-values}, this implies that, for all $x \in \cX$, 
\[
    \E\bsb{\I_{K}(X^x_m)}
=
    \E\Bsb{\babs{\I_{K}(X^x_m)-0}} \to 0 \;, 
    \qquad \text{as } m \to \infty\;,
\]
and then also
\[
    \E\bsb{\I_{K^c}(X^x_m)} = 1-\E\bsb{\I_{K}(X^x_m)} \to 1 \;, 
    \qquad \text{as } m \to \infty\;.
\]
Thus, the Freezing~lemma (Lemma~\ref{l:freezing}) and Lebesgue's dominated convergence theorem yield
\begin{align*}
\P \brb{ Y_{\Smn_m} \neq Y } 
&\overset{\phantom{\eqref{e:freezing}}}{=} \P \brb{ \lcb{ Y_{\Smn_m} = 0 } \cap \lcb{ Y = 1 } } + \P \brb{ \lcb{ Y_{\Smn_m} = 1 } \cap \lcb{ Y = 0 } } \\
&\overset{\phantom{\eqref{e:freezing}}}{=} \P \brb{ \lcb{ X_{\Smn_m} \in K^c } \cap \lcb{ X \in K } } + \P \brb{ \lcb{ X_{\Smn_m} \in K } \cap \lcb{ X \in K^c } } \\
&\overset{\phantom{\eqref{e:freezing}}}{=} \E \bsb{ \I_{K^c}(X_{\Smn_m}) \I_{K}(X) } + \E \bsb{ \I_{K}(X_{\Smn_m}) \I_{K^c}(X) }\\
&\overset{\phantom{\eqref{e:freezing}}}{=} \E\Bsb{\E \bsb{ \I_{K^c}(X_{\Smn_m}) \I_{K}(X) \mid X} } + \E\Bsb{\E \bsb{ \I_{K}(X_{\Smn_m}) \I_{K^c}(X) \mid X} }\\
&\overset{\eqref{e:freezing}}{=} \E\bbsb{\Bsb{\E \bsb{ \I_{K^c}(X^x_m) \I_{K}(x) } }_{x=X}} + \E\bbsb{\Bsb{\E \bsb{ \I_{K}(X^x_m) \I_{K^c}(x) } }_{x=X}}\\
&\overset{\phantom{\eqref{e:freezing}}}{=} \int_{\cX} \E \bsb{ \I_{K^c}(X^x_m) \I_{K}(x) } \dif \P_X(x) + \int_{\cX} \E \bsb{ \I_{K}(X^x_m) \I_{K^c}(x) } \dif \P_X(x) \\
&\overset{\phantom{\eqref{e:freezing}}}{=} \int_{\cX} \I_{K}(x) \E \bsb{ \I_{K^c}(X^x_m) } \dif \P_X(x) + \int_{\cX} \I_{K^c}(x) \E \bsb{ \I_{K}(X^x_m) } \dif \P_X(x) \\
&\qquad \to \P_X(K) + 0 = \P_X(K) = \mu(K) \;, 
\qquad \text{as } m \to +\infty.
\end{align*}
Now, note that $\I_K(X) = \E[Y \mid X]$, and so we can choose $\I_K\aaa{(X)}$ as the regression function $\eta$, leading us to $2 \aaa{\E} \lsb{ \eta(X) \brb{1-\eta(X)} } = 2 \aaa{\E} \lsb{ \I_K(X)\I_{K^c}(X) } = 0 < \mu (K)$, which implies that
\[
    \P \brb{ Y_{\Smn_m} \neq Y } \not\to 2 \E \lsb{ \eta(X) \brb{1-\eta(X)} } \;, \qquad \text{as } m \to \infty.
\]
Therefore, Corollary~\ref{c:const-realizz} does not hold without requiring that $\P_X$-almost every $x\in \cX$ is a \leb{} point for $\eta$ with respect to $\P_X$.
\end{example}

\section*{Acknowledgments}
An earlier version of this work appeared in Roberto Colomboni's master's thesis, written under the supervision of Nicol\`o Cesa-Bianchi.
Both Tom and Rob gratefully acknowledge Nicol\`o's helpful advice. 
Rob also thanks Guglielmo Beretta for the many enlightening discussions.
This work has benefitted from the AI Interdisciplinary Institute ANITI. ANITI is funded by the French ``Investing for the Future – PIA3'' program under the Grant agreement n. ANR-19-PI3A-0004.\footnote{\url{https://aniti.univ-toulouse.fr/}}

\bibliographystyle{plainnat}
\bibliography{biblio}

\appendix

\section{Useful Probabilistic Results}
\label{s:proba}

In this section we present two useful probability lemmas that we used several times throughout the paper.
The first one is needed to avoid relying on conditional probabilities to obtain independence properties (allowing us to state results in non-separable metric spaces).
The second one is the classic ``freezing lemma''.
\begin{lemma}
\label{l:fake-indep}
Let $(\Omega, \cF, \P)$ be a probability space. Let $(\cV, \cF_{\cV})$ and $(\cW, \cF_{\cW})$ be two measurable spaces. Let $U\colon \Omega \to [0,\iop]$, $f\colon \cV \times \cW \to [0,\iop]$, $V\colon \Omega \to \cV$, $W \colon \Omega \to \cW$ be four measurable functions. If $(U,V)$ and $W$ are $\P$-independent, then
\begin{equation}
    \label{e:generalized-conditional-independence}
    \E\bsb{  U f(V, W) \mid V }
=
    \E[ U \mid V ] \, \E \bsb{ f(V, W) \mid V } \;.
\end{equation}
\end{lemma}
\begin{proof}
If $Z$ is a random variable, we will denote by $\sigma(Z)$ the $\sigma$-algebra generated $Z$.
Since $\E[ U \mid V ] \, \E \bsb{ f(V, W) \mid V }$ is a $\sigma(V)$-measurable and non negative random variable, then, by definition of conditional expectation, we only need to prove that, for all $A \in \cF_{\cV}$, we have
\begin{equation}
    \label{e:fake-indep}
    \E \Bsb{ \I_A(V) \, \E[ U \mid V ] \, \E \bsb{ f(V, W) \mid V } }
=
    \E \bsb{ \I_A(V) \, U f(V, W) }\;.
\end{equation}
Assume first  that $U = \I_F$, for some $F \in \cF$.

We begin by further assuming that for all $(v,w) \in \cV \times \cW $, $f(v,w) = \I_{B}(v) \, \I_{C}(w)$, for some $B \in \cF_{\cV}$ and $C \in \cF_{\cW}$. For each $A \in \cF_{\cV}$, we have
\begin{align*}
    \E \Bsb{ \I_A(V) \, \E[ U \mid V ] \, \E \bsb{ f(V, W) \mid V } }
&=
    \E \Bsb{ \I_A(V) \, \E[ \I_F \mid V ] \, \E \bsb{ \I_B(V) \, \I_C(W) \mid V } }
\\
&=
    \E \Bsb{ \I_A(V) \, \I_B(V) \, \E[ \I_F \mid V ] \, \E \bsb{ \I_C(W) \mid V } }
\\
&=
    \E \Bsb{ \I_A(V) \, \I_B(V) \, \E[ \I_F \mid V ] \, \E \bsb{ \I_C(W) } }
\\
&=
    \E \bsb{ \I_{A \cap B}(V) \, \E[ \I_F \mid V ] } \, \P (W \in C)
\\
&=
    \E \bsb{ \I_{A \cap B}(V) \, \I_F } \, \P (W \in C)
\\
&=
    \P \brb{ \{V \in A\} \cap \{V \in B\} \cap F } \, \P (W \in C)
\\
&=
    \P \brb{ \{V \in A\} \cap \{V \in B\} \cap F  \cap \{ W \in C \} }
\\
&=
    \E \bsb{ \I_A(V) \, \I_F \, \I_B(V) \, \I_C(W)  }
\\
&=
    \E \bsb{ \I_A(V) \, U f(V, W)  }\;.
\end{align*}
This proves \eqref{e:fake-indep} under these assumptions.

Then assume that, for all $(v,w) \in \cV \times \cW$
\[
    f(v,w)
=
    \sum_{i=1}^n a_i \, \I_{B_i}(v) \, \I_{C_i}(w)
\]
for some $n \in \N$, $a_1, \l, a_n > 0$, $B_1, \l, B_n \in \cF_{\cV}$, and $C_1, \l, C_n \in \cF_{\cW}$.
For each $A \in \cF_{\cV}$, we have
\begin{multline*}
    \E \Bsb{ \I_A(V) \, \E[ U \mid V ] \, \E \bsb{ f(V, W) \mid V } }
\\
\begin{aligned}
&=
    \E \lsb{ \I_A(V) \, \E[ \I_F \mid V ] \, \E \lsb{ \sum_{i=1}^n a_i \, \I_{B_i}(V) \, \I_{C_i}(W) \mid V } }
\\
&=
    \sum_{i=1}^n a_i \, \E \Bsb{ \I_A(V) \, \E[ \I_F \mid V ] \, \E \bsb{ \I_{B_i}(V) \, \I_{C_i}(W) \mid V } }
\\
&=
    \sum_{i=1}^n a_i \, \E \bsb{ \I_A(V) \, \I_F \, \I_{B_i}(V) \, \I_{C_i}(W) }
\\
&=
    \E \lsb{ \I_A(V) \, \I_F  \sum_{i=1}^n a_i \,  \I_{B_i}(V) \, \I_{C_i}(W) }
\\
&=
    \E \bsb{ \I_A(V) \, U f(V, W)  } \;,
\end{aligned}
\end{multline*}
where the third equality follows by \eqref{e:fake-indep}, which is true in this case for what we proved above. This proves \eqref{e:fake-indep} under these assumptions.

Next, assume that $f = \I_D$, where $D$ belongs to the product $\sigma$-algebra $\cF_{\cV} \otimes \cF_{\cW}$.
Let $\cG$ be the algebra generated by $\Pi:=\{B \times C \mid B \in \cF_{\cV}, C \in \cF_{\cW} \}$.
By \citep[Theorem 6.3, Chapter 6 (The general integral/Approximations)]{lang2012real}, there exists a sequence $(f_n)_{n\in\N}$ such that for all $n\in \N$, there exist $m_n \in \N$, $a_{1,n},...,a_{m_n,n}>0$,  $G_{1,n},...,G_{m_n,n} \in \cG$ such that $f_n = \sum_{k=1}^{m_n} a_{k,n} \, \I_{G_{k,n}} $ and $\lno{ f-f_n }_{\cL^1(\P_{(V,W)})} \to 0$, as $n \to \iop$, i.e. $\E\Bsb{ \babs{ f(V,W)-f_n(V,W) } } \to 0$, as $n \to \iop$. 
Since $\Pi$ is a $\pi$-system, we have that the elements of $\cG$ are finite unions of disjoint elements of $\Pi$, and so for each $n \in \N$ and each $k \in \{1,...,m_n\}$ there \aaa{exist} $l_{n,k} \in \N$ and $B_{1,n,k} \times C_{1,n,k}, \l, B_{l_{n,k},n,k} \times C_{l_{n,k},n,k} \in \Pi$ mutually disjoint such that $G_{k,n} = \bigcup_{j=1}^{l_{n,k}} B_{j,n,k} \times C_{j,n,k}$. 
Therefore for each $n \in \N$
\[
    f_n 
= 
    \sum_{k=1}^{m_n} a_{k,n} \, \I_{G_{k,n}} 
= 
    \sum_{k=1}^{m_n} a_{k,n} \, \I_{\bigcup_{j=1}^{l_{n,k}} B_{j,n,k} \times C_{j,n,k}} 
= 
    \sum_{k=1}^{m_n} \sum_{j=1}^{l_{n,k}} a_{k,n} \, \I_{ B_{j,n,k} \times C_{j,n,k}}\;.
\]
Then, $\E \bbsb{ \Babs{ \E \bsb{ f_n(V,W) \mid V }  - \E\bsb{ f(V,W) \mid V } } } \to 0$, as $n\to \iop$.
Similarly, since $\I_F$ is bounded, $\E \bbsb{ \Babs{ \E\bsb{\I_F f_n(V,W) \mid V } - \E\bsb{ \I_F f(V,W) \mid V } } } \to 0$ as $n\to \iop$.
Moreover, being $\I_F$ bounded, its conditional expectation $\E [ \I_F\mid V ]$ is also bounded, which in turn yields $\E \bbsb{ \Babs{ \E\bsb{\I_F \mid V} \E\bsb{ f_n(V,W) \mid V }  - \E\bsb{\I_F \mid V} \E\bsb{ f(V,W) \mid V } } } \to 0$ as $n\to \iop$.
Thus, for each $A \in \cF_{\cV}$,
\begin{align*}
    \E \Bsb{ \I_A(V) \, \E[ U \mid V ] \, \E \bsb{ f(V, W) \mid V } }
&=
    \lim_{n \to \infty} \E \Bsb{ \I_A(V) \, \E[ \I_F \mid V ] \, \E \bsb{ f_n(V, W) \mid V } }
\\
&=
    \lim_{n \to \infty} \E \Bsb{ \I_A(V) \, \I_F \, f_n(V,W) }
\\
&=
    \E \bsb{ \I_A(V) \, U f(V, W)  } \;\aaa{,}
\end{align*}
where the third equality follows by \eqref{e:fake-indep}, which is true in this case for what we proved above. This proves \eqref{e:fake-indep} under these assumptions.

Let now $f = \sum_{i=1}^n a_i \, \I_{D_i}$, for some $n \in \N$, $a_1, \l, a_n > 0$, and $D_1, \l, D_n \in \cF_{\cV} \otimes \cF_{\cW}$. Then, for each $A \in \cF_{\cV}$,
\begin{multline*}
    \E \Bsb{ \I_A(V) \, \E[ U \mid V ] \, \E \bsb{ f(V, W) \mid V } }
\\
\begin{aligned}
&=
    \E \lsb{ \I_A(V) \, \E[ \I_F \mid V ] \, \E \lsb{ \sum_{i=1}^n a_i \, \I_{D_i}(V,W) \mid V } }
\\
&=
    \sum_{i=1}^n a_i \, \E \Bsb{ \I_A(V) \, \E[ \I_F \mid V ] \, \E \bsb{ \I_{D_i}(V,W) \mid V } }
\\
&=
    \sum_{i=1}^n a_i \, \E \bsb{ \I_A(V) \, \I_F \, \I_{D_i}(V,W) }
\\
&=
    \E \lsb{ \I_A(V) \, \I_F  \sum_{i=1}^n a_i \, \I_{D_i}(V,W) }
\\
&=
    \E \bsb{ \I_A(V) \, U f(V, W)  } \;,
\end{aligned}
\end{multline*}
where the third equality follows by \eqref{e:fake-indep}, which is true in this case for what we proved above. This proves \eqref{e:fake-indep} under these assumptions.

Now, if $f$ is general, we can get a sequence $(f_n)_{n\in \N}$ such that for each $n \in \N$ there exists $m_n \in \N$, $a_{1,n}, \l, a_{m_n,n} > 0$, and $D_{1,n},\l,D_{m_n,n} \in \cF_{\cV} \otimes \cF_{\cW}$ such that for each $n \in \N$ we have that $f_n = \sum_{k=1}^{m_n} a_{k,n} \I_{D_{k,n}}$ and $f_n \uparrow f$ pointwise, as $n \uparrow \iop$. Hence, by the monotone convergence theorem for the conditional expectation, we have that, for each $A \in \cF_{\cV}$,
\begin{align*}
    \E \Bsb{ \I_A(V) \, \E[ U \mid V ] \, \E \bsb{ f(V, W) \mid V } }
&=
    \lim_{n \to \infty} \E \Bsb{ \I_A(V) \, \E[ \I_F \mid V ] \, \E \bsb{ f_n(V, W) \mid V } }
\\
&=
    \lim_{n \to \infty} \E \Bsb{ \I_A(V) \, \I_F \, f_n(V,W) }
\\
&=
    \E \bsb{ \I_A(V) \, U f(V, W)  } \;.
\end{align*}
Now, suppose $U = \sum_{i=1}^n a_n \I_{F_i}$ for some $n \in \N$, for $n$ distinct $a_1, \l, a_n > 0$ and $F_1,...,F_n \in \cF$. For each $i \in \{1, \l, n\}$, we have that $F_i = \{U = a_i\}$ so $\sigma(\I_{F_i},V) \subset \sigma(U,V)$, and since $(U,V)$ is $\P$-independent from $W$ we also have that $(\I_{F_i},V)$ is $\P$-independent from $W$. Then, for each $A \in \cF_{\cV}$, we have that
\begin{multline*}
    \E \Bsb{ \I_A(V) \, \E[ U \mid V ] \, \E \bsb{ f(V, W) \mid V } }
\\
\begin{aligned}
&=
    \E \lsb{ \I_A(V) \, \E \lsb{ \sum_{i=1}^n a_i \, \I_{F_i} \mid V } \, \E \lsb{ f(V, W) \mid V } }
\\
&=
    \sum_{i=1}^n a_i \, \E \Bsb{ \I_A(V) \, \E[ \I_{F_i} \mid V ] \, \E \bsb{ f(V,W) \mid V } }
\\
&=
    \sum_{i=1}^n a_i \, \E \bsb{ \I_A(V) \, \I_{F_i} \, f(V,W) }
\\
&=
    \E \lsb{ \I_A(V) \, \sum_{i=1}^n a_i \, \I_{F_i} \, f(V,W) }
\\
&=
    \E \bsb{ \I_A(V) \, U f(V, W)  } \;,
\end{aligned}
\end{multline*}
where in the third equality we used the previous case.
Finally, if $U$ is general, for each $n\in \N$ and $i\in\{ 1,\l,2^{2n-1}-1 \}$ define $a_{i,n} = \frac{i 2^n}{2^{2n-1}}$ and $F_{i,n} = \lcb{U \in \bigl[ \frac{i 2^n}{2^{2n-1}}, \frac{(i+1) 2^n}{2^{2n-1}}\bigr) }$. For each $n\in \N$, define
\[
    U_n = \sum_{i=1}^{2^{2n-1}} a_{i,n} \I_{F_{i,n}}\;.
\]
Then, for each $n\in\N$, we have that $a_{1,n}, \l, a_{2^{2n}-1,n}>0$ are distinct, that $F_{1,n},\l,F_{2^{2n}-1,n} \in \cF$ are mutually disjoint. Also, for each $n \in \N$, we have that $\sigma(U_n) \subset \sigma \brb{\I_{F_{1,n}}, \l, \I_{F_{2^{2n}-1,n}} } \subset \sigma (U)$ and so $(U_n,V)$ is $\P$-independent from $W$ since $(U,V)$ is $\P$-independent from $W$. Furthermore, we have that $U_n \uparrow U$ pointwise as $n \uparrow \iop$ and so, since $f(V,W)\ge 0$, also $U_n f(V,W) \uparrow U f(V,W)$ pointwise as $n \uparrow \iop$ and $\E[ U_n \mid V ] \, \E \bsb{ f(V, W) \mid V } \uparrow \E[ U \mid V ] \, \E \bsb{ f(V, W) \mid V }, \P$-almost everywhere as $n \uparrow \iop$. Hence, by what we observed and the monotone convergence theorem, we have that, for each $A \in \cF_{\cV}$,
\begin{align*}
    \E \Bsb{ \I_A(V) \, \E[ U \mid V ] \, \E \bsb{ f(V, W) \mid V } }
&=
    \lim_{n \to \infty} \E \Bsb{ \I_A(V) \, \E[ U_n \mid V ] \, \E \bsb{ f(V, W) \mid V } }
\\
&=
    \lim_{n \to \infty} \E \Bsb{ \I_A(V) \, U_n \, f(V,W) }
\\
&=
    \E \bsb{ \I_A(V) \, U f(V, W)  } \;.
\end{align*}
where in the second equality we used the previous case.
This concludes the proof.
\end{proof}

The next \aaa{result} can be proven with the same approach as the previous lemma. Alternatively, a proof is given in \citep[Lemma 4.1]{baldi2017stochastic}.

\begin{lemma}[The ``freezing lemma'']
\label{l:freezing}
Let $(\Omega, \cF, \P)$ be a probability space. Let $(\cV, \cF_{\cV})$ and $(\cW, \cF_{\cW})$ be two measurable spaces. Let $f\colon \cV \times \cW \to [0,\iop]$, $V\colon \Omega \to \cV$, $W \colon \Omega \to \cW$ be three measurable functions. If $V$ and $W$ are $\P$-independent, then
\begin{equation}
    \label{e:freezing}
    \E \bsb{ f(V,W) \mid V }
=
     \Bsb{ \E \bsb{ f(v,W) } }_{v=V}
\end{equation}
$\P$-almost surely, where the right hand side is the composition
\[
    \Bsb{ \E \bsb{ f(v,W) } }_{v=V} 
= 
    \Brb{ v\mapsto \E \bsb{ f(v,W) } } \circ V \;.
\]
\end{lemma}

\end{document}